\documentclass[letterpaper]{article} 
\usepackage{aaai2026}  
\usepackage{times}  
\usepackage{helvet}  
\usepackage{courier}  
\usepackage[hyphens]{url}  
\usepackage{graphicx} 
\usepackage{natbib}  
\usepackage{caption} 
\usepackage{algorithm}
\usepackage{multicol}

\usepackage{newfloat}
\usepackage{listings}
\DeclareCaptionStyle{ruled}{labelfont=normalfont,labelsep=colon,strut=off} 
\floatstyle{ruled}
\newfloat{listing}{tb}{lst}{}
\floatname{listing}{Listing}
\usepackage{booktabs}       
\usepackage{amsfonts}       
\usepackage{nicefrac}       
\usepackage{microtype}      
\usepackage{xcolor}         
\usepackage{makecell}       
\usepackage{enumitem}       
\usepackage{multirow}       
\usepackage[table]{xcolor}  
\usepackage{amssymb}        
\usepackage{amsthm}         
\usepackage{bm}             
\usepackage{pifont}         
\usepackage[utf8]{inputenc}
\usepackage[T1]{fontenc}
\usepackage{url}
\usepackage{algpseudocode}
\usepackage{amsmath}

\newtheorem{proposition}{Proposition}

\title{Dropout Prompt Learning: Towards Robust and Adaptive Vision-Language Models}
\author{
    Biao Chen,
    Lin Zuo\thanks{Corresponding author.},
    Mengmeng Jing, Kunbin He, Yuchen Wang
}
\affiliations{
    School of Information and Software Engineering,\\University of Electronic Science and Technology of China\\

    chenbiao@std.uestc.edu.cn, linzuo@uestc.edu.cn, \{jingmeng1992, hekunbin19, indramoe472\}@gmail.com 
}

\usepackage{bibentry}

\begin{document}
\maketitle

\begin{abstract}
Dropout is a widely used regularization technique which improves the generalization ability of a model by randomly dropping neurons. In light of this, we propose Dropout Prompt Learning, which aims for applying dropout to improve the robustness of the vision-language models. Different from the vanilla dropout, we apply dropout on the tokens of the textual and visual branches, where we evaluate the token significance considering both intra-modal context and inter-modal alignment, enabling flexible dropout probabilities for each token. Moreover, to maintain semantic alignment for general knowledge transfer while encouraging the diverse representations that dropout introduces, we further propose residual entropy regularization. Experiments on 15 benchmarks show our method's effectiveness in challenging scenarios like low-shot learning, long-tail classification, and out-of-distribution generalization. Notably, our method surpasses regularization-based methods including KgCoOp by 5.10\% and PromptSRC by 2.13\% in performance on base-to-novel generalization. Our code is available at \url{https://github.com/JustCoolPig/DroPLe}.
\end{abstract}
    
\section{Introduction}
\label{sec:intro}
Vision-Language Models (VLMs) such as CLIP~\cite{radford2021learning} and ALIGN~\cite{jia2021scaling} have achieved remarkable advantages in zero-shot scenarios. While prompt learning~\cite{zhou2022learning, khattak2023maple} offers a parameter-efficient approach for adapting pre-trained VLMs to downstream tasks, its generalization capability remains limited by overfitting issues, particularly in low-data scenarios~\cite{park2024prompt,khattak2023self}.

\begin{figure}
    \centering    \includegraphics[width=1\linewidth]{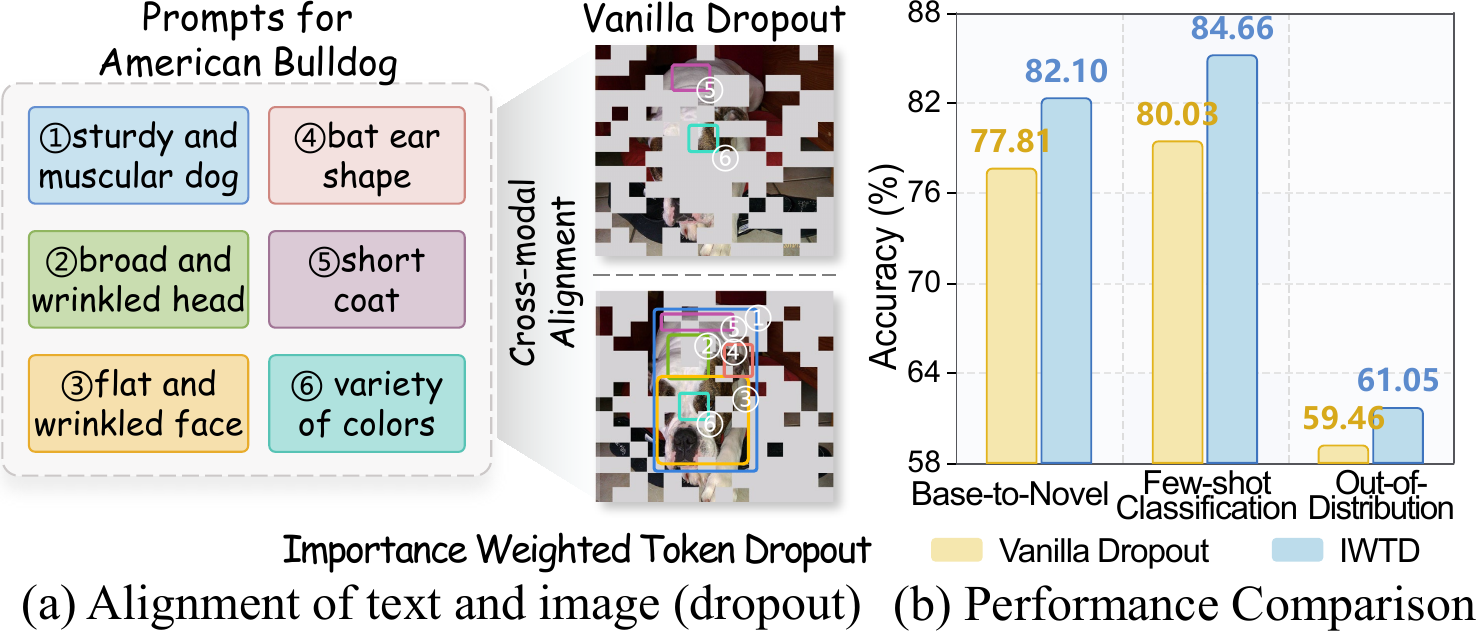}
    \caption{(a) Vanilla dropout randomly removes visual tokens, disrupting image-text alignment (top). Importance Weighted Token Dropout preserves semantically relevant tokens for alignment (bottom). (b) Comparison on base-to-novel, out-of-distribution generalization and few-shot image classification.}
    \label{motivation}
\end{figure}
In the past decade, dropout is applied as an effective regularization technique in deep neural networks, significantly mitigating overfitting and improving generalization by randomly dropping neurons~during training~\cite{srivastava2014dropout}. Dropout prevents complex co-adaptations among feature detectors and implicitly averages over an exponential number of thinned network architectures, a critical factor in the success of models such as AlexNet~\cite{krizhevsky2012imagenet}. While dropout has shown remarkable success across various deep learning architectures, its potential in prompt learning for VLMs remains unexplored. Motivated by the effectiveness of dropout in learning robust models, we propose to incorporate dropout mechanisms into VLM prompt learning to enhance model generalization, particularly in low-data regimes~\cite{zhou2022conditional,zhu2023prompt,chen2025chain}.

However, VLM prompt learning presents distinct challenges compared to traditional deep learning, raising three critical questions regarding dropout implementation:
(1) \textit{Where to drop}: VLMs rely on tokens as fundamental semantic units to facilitate fine-grained semantic alignment across modalities (e.g., the contrastive learning mechanism of CLIP~\cite{radford2021learning}), while vanilla dropout would destroy this alignment.
As shown in Fig.~\ref{motivation}(a), randomly dropping critical visual tokens impairs their matching with textual descriptions, leading to degraded performance (Fig.~\ref{motivation}(b)). While existing works in unimodal tasks~\cite{ke2020group,zhai2018adaptive} improve the vanilla dropout through adaptive probabilities, these approaches are unsuitable for VLMs requiring cross-modal dependencies.
(2) \textit{What degree to drop}: Unlike traditional neural networks where high parameter redundancy enables effective dropout without performance loss, VLMs process semantically dense tokens with limited token-level redundancy. The inherent token sparsity means that high dropout ratios on semantically rich tokens could severely degrade performance, while low ratios on less informative tokens provide insufficient regularization. This creates a challenge in determining optimal dropout scheduling that balances feature preservation with regularization.
(3) \textit{How to learn from dropout}: To prevent semantic drift between learnable and frozen branches, existing approaches~\cite{yao2023visual,khattak2023self} often enforce strict $\mathcal{L}_1$ or $\mathcal{L}_2$ regularization. This operation, however, is overly strict for dropout prompt learning. Such strict constraints limit the benefits from dropout-induced variations, suggesting the need for a mechanism that balances consistency and diversity.

Towards more robust and general prompt learning, we propose \textbf{Dropout Prompt Learning}, a principled framework that incorporates dropout mechanisms into vision-language prompt learning by regularizing through token dropout. Based on the framework of dropout prompt learning, we present the \textit{Importance Weighted Token Dropout}, termed as IWTD, which formulates dropout as a token importance estimation problem in the multimodal space. 

Importance weighted token dropout is carefully designed to handle the three challenges. For the first challenge of \textit{where to drop}, we leverage a comprehensive importance metric to jointly model intra-modal context, inter-modal alignment, and task-specific relevance through a unified attention mechanism. This enables the identification of semantically critical tokens that maintain cross-modal alignment.
For the second challenge of \textit{what degree to drop}, different samples exhibit varying semantic densities in their tokens. Tokens carrying minimal semantic information can tolerate higher dropout rates for enhancing generalization, while samples with high semantic density require lower dropout rates to preserve crucial tokens for cross-modal alignment. This motivates flexible dropout probability assignment according to token significance.
For the third challenge of \textit{how to learn from dropout}, we propose residual entropy regularization, which computes residuals between pre- and post-dropout feature representations, and maximizes the predictive entropy on these residuals, simultaneously maintaining alignment with general knowledge transfer while encouraging representational diversity.
Our main contributions are as follows:
\begin{itemize}
  \item We propose Dropout Prompt Learning, a novel learning paradigm that extends dropout regularization to vision-language model adaptation. By introducing token-level dropout strategies, this framework enhances model generalization ability while maintaining cross-modal alignment.
  \item We present importance weighted token dropout, an effective implementation of dropout prompt learning. It dynamically adjusts dropout probabilities by jointly considering intra-modal context and cross-modal alignment. The residual entropy regularization is further adopted to maintain semantic alignment for general knowledge transfer while encouraging diverse feature representations.
  \item Extensive experiments on 15 benchmark datasets comprehensively validate the robustness and superior performance of the proposed method under various challenging settings, including low-shot learning, long-tail classification, and out-of-distribution generalization.
\end{itemize}




\section{Related Work}
\label{related_work}
\noindent\textbf{Prompt Learning in VLMs. }
Prompt learning has evolved from NLP~\cite{li2021prefix} to VLMs~\cite{lu2022prompt,zhou2022conditional}, with CoOp~\cite{zhou2022learning} introducing soft prompts for CLIP. Recent advances explore multimodal prompting~\cite{khattak2023maple,cho2023distribution}. However, limited data often leads to overfitting~\cite{khattak2023self,park2024prompt}, inspiring various regularization methods: ProGrad~\cite{zhu2023prompt} aligns prompt gradients with general knowledge. KgCoOp~\cite{yao2023visual} minimizes distance between learned and hand-crafted embeddings. PSRC~\cite{khattak2023self} uses self-regularization through mutual agreement, and ProMetaR~\cite{park2024prompt} employs meta-learning with task augmentation.
GalLoP~\cite{lafon2024gallop} applies dropout on multiple complete candidate texts to enhance diversity, yet this coarse-grained operation differs from dropout's principle of fine-grained dropping of individual units.
Despite these approaches, the potential of dropout in the prompt learning remains underexplored. To address this limitation, we propose a token-level adaptive dropout framework for robust prompt learning in VLMs.

\paragraph{Dropout regularization.}
Vanilla Dropout~\cite{srivastava2014dropout} and its variants like DropConnect~\cite{wan2013regularization}, DropBlock~\cite{ghiasi2018dropblock}, and Curriculum Dropout~\cite{morerio2017curriculum} apply fixed dropout probabilities during training. However, static dropout rates cannot adapt to varying feature importance across inputs and layers. This limitation motivates adaptive dropout methods that dynamically adjust probabilities. StandOut~\cite{ba2013adaptive} pioneered input-dependent rate learning via auxiliary networks, while subsequent work leveraged Rademacher complexity~\cite{zhai2018adaptive} and feature distributions~\cite{ke2020group}. Recent works include attention mechanisms~\cite{yang2022ad} and GFlowNet~\cite{liu2023gflowout} which learns data-dependent dropout masks via posterior inference. However, adaptive dropout remains underexplored in multimodal settings such as VLMs. Thus, we propose a multimodal importance metric that jointly considers intra-modal context and inter-modal alignment for adaptive dropout. We also compare with unimodal adaptive dropout in Appendix.

\paragraph{Consistency regularization.}
Consistency regularization preserves model generalization by minimizing discrepancies between learned and reference features. 
Common approaches maintain consistency via feature alignment using $\mathcal{L}_1$ or $\mathcal{L}_2$ norms~\cite{laine2016temporal,tarvainen2017mean,sajjadi2016regularization} or cosine similarity~\cite{hoe2021one}, or through distribution matching with KL divergence~\cite{li2018conversational}. These methods have proven effective in enhancing model robustness~\cite{wang2021regularizing}.
Consistency regularization also effectively mitigates overfitting and knowledge forgetting in VLM prompt learning. Methods like KgCoOp~\cite{yao2023visual}, PSRC~\cite{khattak2023self}, and CoPrompt~\cite{CoPrompt} work by constraining learnable prompts using references such as frozen CLIP features.
While beneficial for generalization, enforcing strict consistency between multi-source features often limits model flexibility. Our residual entropy regularization relaxes this strict constraint, enabling the model to balance semantic alignment with the diverse representations introduced by adaptive dropout.






\begin{figure*}
    \centering    \includegraphics[width=0.85\linewidth]{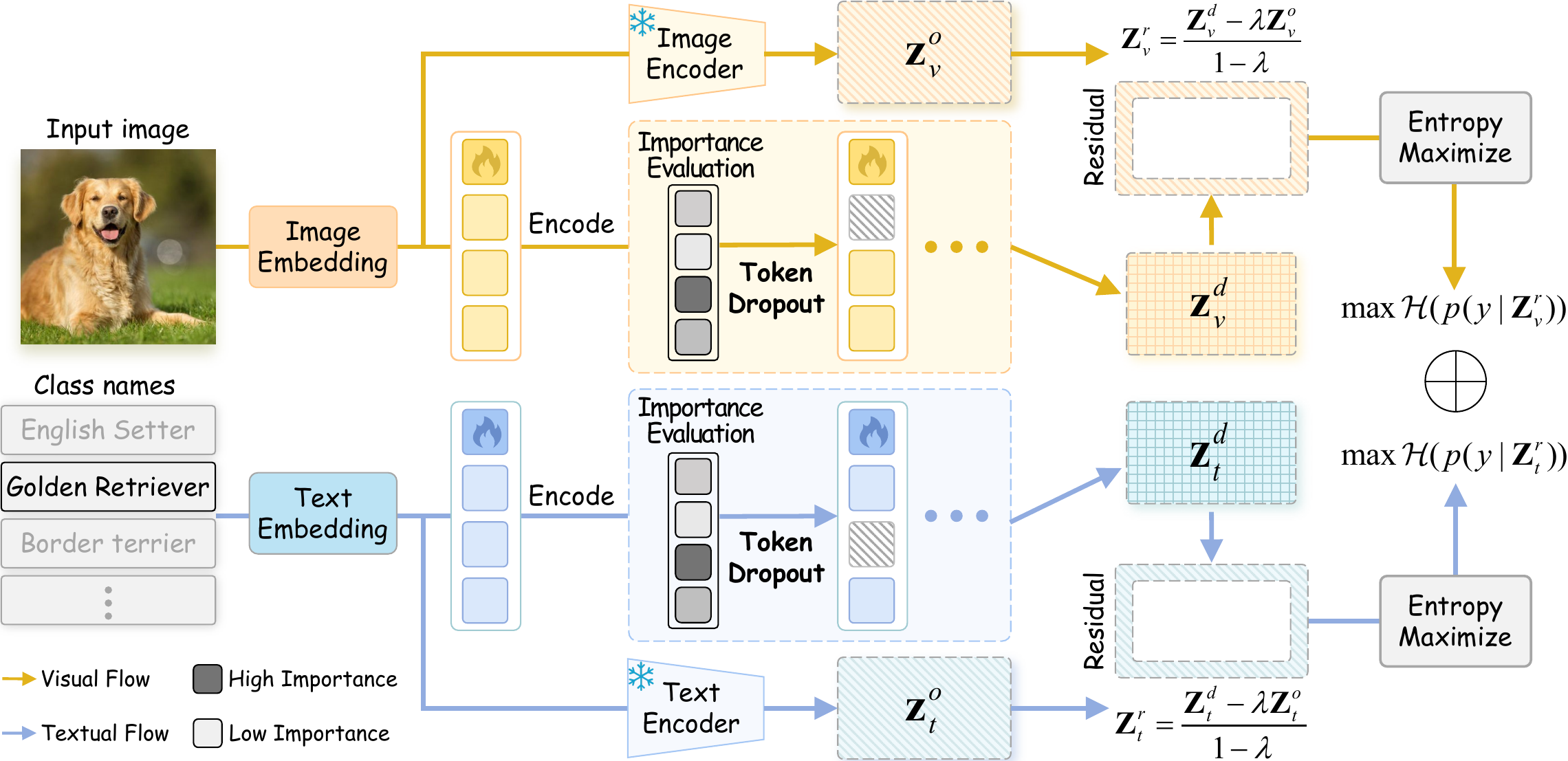}
    \caption{\textbf{Method overview of Importance Weighted Token Dropout.} Textual and visual modalities are processed by parallel encoding pathways, a frozen branch and a learnable branch. In the learnable branch, we compute an intra-/inter-modal importance metric for tokens at each layer, which guides adaptive token dropout. Then, residual features derive from learnable and frozen branch differences. Finally, maximizing entropy constrains dropout for both visual and textual residuals.}
    \label{pipeline}
\end{figure*}
\section{Methodology}\label{method}
\subsection{Preliminaries}\label{Preliminaries}
\paragraph{CLIP.}
Contrastive Language-Image Pre-training (CLIP) \cite{radford2021learning} uses image ($\mathcal{F}$) and text ($\mathcal{G}$) encoders to produce aligned features ($\mathbf{z}_v = \mathcal{F}(x), \mathbf{z}_t = \mathcal{G}(t)$) via contrastive loss.
For K-way zero-shot classification, image features $\mathbf{z}_v$ are compared against $K$ class text features $\{\mathbf{z}_{t,k}\}_{k=1}^K$ using cosine similarity and a softmax function with temperature $\tau$ to compute probabilities $p(y=k|x)$:
\begin{equation}
p(y=k|x) = \frac{\exp(\text{sim}(\mathbf{z}_v, \mathbf{z}_{t,k}) / \tau)}{\sum_{j=1}^K \exp(\text{sim}(\mathbf{z}_v, \mathbf{z}_{t,j}) / \tau)}.
\label{clip}
\end{equation}
\paragraph{Prompt learning.}
Prompt learning optimizes VLMs by incorporating learnable prompts instead of full fine-tuning \cite{zhou2022learning,khattak2023maple}. The text and visual input sequences at layer $i$ are defined as:
$T_{input}^{(i)} = \{t_{bos}, P_{t}^{(i)}, T_{embed}, t_{eos}\}$ and 
$V_{input}^{(i)} = \{v_{cls}, E_{patch}, P_{v}^{(i)}\}$, 
where $P_{t}^{(i)} = \{p_t^1, p_t^2, ..., p_t^\eta\}$ and $P_{v}^{(i)} = \{p_v^1, p_v^2, ..., p_v^M\}$ are learnable prompts with dimensions $\mathbb{R}^{\eta}$ and $\mathbb{R}^{M}$.

\paragraph{Adaptive Dropout.}
Adaptive dropout dynamically adjusts dropout probabilities based on feature importance rather than fixed rates. For unit $i$ with importance score $I_i$, the dropout probability is $p_i = f(I_i)$, where $f(\cdot)$ maps importance scores to dropout probabilities. The dropout operation then follows:
\begin{equation}
    \boldsymbol{\mu}_i \sim \text{Bernoulli}(1-p_i), \quad \boldsymbol{\phi}_i = (\mathbb{W}_i\boldsymbol{\theta}) \odot \boldsymbol{\mu}_i/(1-p_i),
\end{equation}
where $1/(1-p_i)$ maintains expected output during training.

\subsection{Dropout Prompt Learning}\label{sec:DroPLe}
Dropout prompt learning extends dropout principles to VLMs by applying dropout at the token level to enhance robustness and generalization. 
Let \scalebox{0.95}{$U^{(i)} \in \{V^{(i)}_{\text{input}}, T^{(i)}_{\text{input}}\}$} denote the visual or textual token sequence at layer $i$, where \scalebox{0.95}{$V^{(i)}_{\text{input}} = \{v_{\text{cls}}, E_{\text{patch}}, P^{(i)}_v\}$} and \scalebox{0.95}{$T^{(i)}_{\text{input}} = \{t_{\text{bos}}, P^{(i)}_t, T_{\text{embed}}, t_{\text{eos}}\}$}. Here, \scalebox{0.95}{$P^{(i)}_v$} and \scalebox{0.95}{$P^{(i)}_t$} are learnable prompts. For training phase, Dropout prompt learning applies the dropout operation $\mathcal{D}_{\text{token}}$ through the following formulation:
\begin{equation}
\begin{aligned}
U_{\text{dropped}}^{(i)} = f_{\text{CLIP}}(&\mathcal{D}_{\text{token}}(\text{Enc}(U^{(i)}; \theta_e) \odot \mathcal{B}(p)_{1 \times n}); \theta_{\text{CLIP}}),
\end{aligned}
\end{equation}
where $\text{Enc}(\cdot; \theta_e)$ is the modality-specific encoder with parameters $\theta_e$, $\mathcal{B}(p)_{1 \times n}$ denotes an $n$-dimensional vector of independent Bernoulli random variables with dropout rate $p$, and $f_{\text{CLIP}}(\cdot; \theta_{\text{CLIP}})$ denotes the CLIP model with parameters $\theta_{\text{CLIP}}$.
During inference, following the standard dropout protocol \cite{srivastava2014dropout}, dropout is disabled and we directly apply the trained model:
\begin{equation}
U_{\text{infer}}^{(i)} = f_{\text{CLIP}}(\text{Enc}(U^{(i)}; \theta_e); \theta_{\text{CLIP}}).
\end{equation}

Nevertheless, vanilla token dropout can disrupt cross-modal semantic alignment in VLMs, potentially degrading performance by randomly dropping critical visual or textual tokens from $U^{(i)}$. To address this, dropout prompt learning requires dropout strategies $\mathcal{D}_{\text{token}}$ designed explicitly for VLMs. We propose importance weighted token dropout (IWTD), which evaluates token significance from multiple perspectives to guide adaptive dropout. As shown in Fig.~\ref{pipeline}, IWTD measures token importance from various aspects and employs these to assign flexible dropout probabilities. Furthermore, it incorporates residual entropy regularization to constrain the adaptive token dropout process.
\subsection{Importance Weighted Token Dropout (IWTD)}\label{sec:iwtd}
Based on the dropout prompt learning framework, we propose importance weighted token dropout. This assigns dropout probabilities via a multimodal importance metric instead of uniform randomness, preserving tokens critical for cross-modal alignment while enabling effective regularization.


\paragraph{Multimodal Importance Metric.}
The metric is termed as $I(\mathbf{x}^{(i)}_j)$, which quantifies the significance of a token $\mathbf{x}^{(i)}_j$ at layer $i$ from multiple sources:
\begin{equation}
    I(\mathbf{x}^{(i)}_j) = f\left( S_{cls}^{(i)}(j), S_{self}^{(i)}(j), S_{cross}^{(i)}(j) \right),
    \label{eq:importance_metric}
\end{equation}
where $j$ indexes the tokens in the sequence of length $L$ at layer $i$, and $f(\cdot)$ is an averaging function. $S_{cls}$, $S_{self}$, and $S_{cross}$ denote the class attention, self-attention and cross-modal attention score respectively, each defined as follows.

First, to capture intra-modal relationships, we introduce the Self-Attention Score ($S_{self}$), which quantifies token interactions within its modality. Given the self-attention tensor $\mathbf{A}_{self}^{(i)} \in \mathbb{R}^{B \times H \times L \times L}$ at layer $i$, where $B$ and $H$ denote batch size and number of attention heads, we compute:
\begin{equation}
S_{self}^{(i)}(j) = \frac{1}{H} \sum_{h=1}^H \max_{k \neq j} \left( (\mathbf{A}_{self}^{(i)})_{b,h,j,k} \right),
\label{eq}
\end{equation}
where token $k$ excludes global tokens $[v_{cls}]$ and $[t_{eos}]$.

Second, to capture intra-modal task-specific importance of tokens, we leverage the attention patterns of the primary task-specific token (i.e., $[v_{cls}]$ for vision and $[t_{eos}]$ for text). The Class Attention Score ($S_{cls}$) is defined as:
\begin{equation}
    S_{cls}^{(i)}(j) = \frac{1}{H} \sum_{h=1}^H \left( (\mathbf{A}_{self}^{(i)})_{b,h,cls,j} \right)
    \label{eq:score_cls}.
\end{equation}

Since task-specific tokens aggregate intra-modal information, their attention reflects each token's task relevance.


Third, in vision-language tasks where cross-modal semantic alignment is essential for understanding inter-modality interactions, we propose the Cross-modal Attention Score ($S_{cross}$) to quantify token importance via cross-modal alignment. Specifically, as shown in Fig.~\ref{mim_metric}, given visual tokens
\scalebox{1.0}{$V_{input}^{(i)} \in \mathbb{R}^{N' \times D_v}$} and textual tokens \scalebox{1.0}{$T_{input}^{(i)} \in \mathbb{R}^{M' \times D_t}$} at the $i$-th layer, we employ linear projections to map these tokens into a shared $d$-dimensional semantic space, yielding $\mathbf{V}'^{(i)} \in \mathbb{R}^{N' \times d}$ and $\mathbf{T}'^{(i)} \in \mathbb{R}^{M' \times d}$ respectively.
To facilitate cross-modal interaction, we introduce $\xi$ learnable bridge tokens $\mathbf{E} \in \mathbb{R}^{\xi \times d}$ as semantic anchors, reducing complexity from direct cross-modal attention $\mathcal{O}(N' \times M')$ to $\mathcal{O}(\xi \times (N' + M'))$. We then compute attention maps $\mathbf{A}^\mathcal{M}_{cross}$ between these bridge tokens and projected modality-specific features. For each modality $\mathcal{M} \in \{v, t\}$ with projected features $\mathbf{X}'^{(i)}$ (denoting either $\mathbf{V}'^{(i)}$ or $\mathbf{T}'^{(i)}$),
\begin{equation}
 \mathbf{A}^\mathcal{M}_{cross} = \text{softmax}\left( \frac{\mathbf{E} (\mathbf{X}'^{(i)})^\top}{\sqrt{d}} \right), \quad \text{for } \mathcal{M} \in \{v, t\}. \label{eq:attn_cross_unified}
\end{equation}

\begin{figure}
    \centering \includegraphics[width=0.48\textwidth]{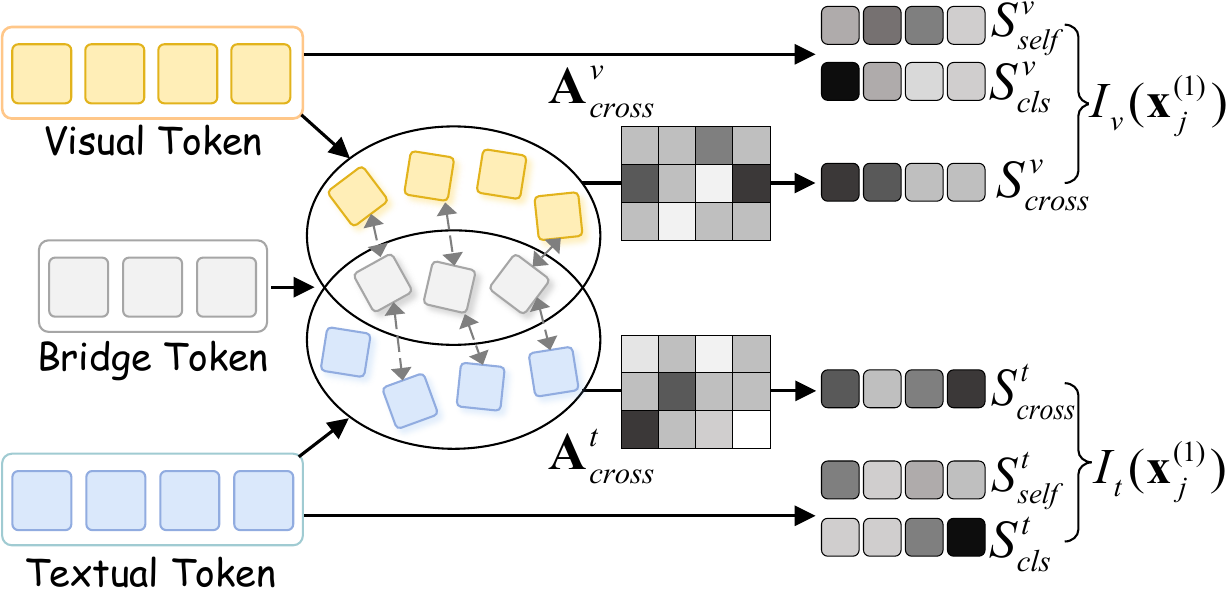}
  \caption{\textbf{Multimodal Importance Metric}, which simultaneously considers both intra-modal attention $S_{self}$, $S_{cls}$ and inter-modal attention $S_{cross}$.}
  \label{mim_metric}
\end{figure}

Based on the computed attention maps, we define the cross-modal importance score $S^\mathcal{M}_{cross}(j)$ for each token $j$ in modality $\mathcal{M}$ as the maximum attention weight received from any bridge token $\kappa$. This score quantifies each token's engagement with the cross-modal bridge tokens:
\begin{equation}
   S^\mathcal{M}_{cross}(j) = \max_{\kappa \in \{1, \dots, \xi\}} \left( (\mathbf{A}^\mathcal{M}_{cross})_{\kappa,j} \right).
\end{equation}

For each modality $\mathcal{M}$, we denote $S^\mathcal{M}_{cross}(j)$ as $S_{cross}^{(i)}(j)$ at layer $i$, which combines with other importance measures to form our final token importance metric $I(\mathbf{x}^{(i)}_j)$. This unified metric guides the adaptive token dropout process.

\paragraph{Token Dropout.}
Let $I(\mathbf{x}^{(i)}_j)$ be the importance score for token $\mathbf{x}^{(i)}_j$ at layer $i$, and $\mathcal{J}_{target}^{(i)}$ denote the indices of token excluding $[v_{cls}]$ and $[t_{eos}]$ which are essential for cross-modal similarity computation. Our objective is to adaptively adjust token dropout, where the dropout probability $p_j$ of each token is inversely mapped from its normalized importance score:
\begin{equation}
    p_j = p_{max} - \hat{I}(\mathbf{x}^{(i)}_j)(p_{max} - p_{min}), \quad j \in \mathcal{J}_{target}^{(i)},
\label{eq:probability_p_j}
\end{equation}
where $\hat{I}(\mathbf{x}^{(i)}_j)$ is the normalized importance score, and $p_{max}$, $p_{min}$ specify the probability bounds. 
For each token representation, we apply the dropout operation $\mathcal{D}(\cdot)$ with its corresponding probability:
\begin{equation}
\hspace{-5pt}
   \mathbf{x}_{out,j}^{(i)} = 
   \begin{cases} 
       \mathcal{D}(\mathbf{x}_{j}^{(i)}) & \text{if } j \in \mathcal{J}_{target}^{(i)} \text{ and } \text{rand}() < p_j \\
       \mathbf{x}_{j}^{(i)} & \text{otherwise} 
   \end{cases}.
\end{equation}

Unlike vanilla dropout that applies uniform randomness globally, our method restricts randomness to importance-determined probability ranges while preserving stochastic sampling within these ranges. This strategy preserves critical tokens for cross-modal alignment while providing effective regularization. To formally characterize IWTD's regularization effect, we analyze its generalization bounds through Rademacher complexity~\cite{zhai2018adaptive,arora2021dropout} (see \textbf{Proposition~\ref{prop:iwtd_rademacher_comparison}} and \textit{proof} in Appendix A.1).

\paragraph{Residual Entropy Regularization.}
Prompt learning methods typically employ consistency regularization to align learnable prompts with pre-trained representations for general knowledge transfer~\cite{yao2023visual,khattak2023self}. Conventional $\mathcal{L}_1$ or $\mathcal{L}_2$ consistency regularization, while mitigating semantic drift between variable and fixed branch, is overly restrictive against beneficial dropout-induced variations. Therefore, we propose residual entropy regularization to permit such advantageous diversity.
Specifically, we denote the output embedding features from the VLM branch processed by IWTD as $\mathbf{z}^d$, and the original VLM output embedding features as $\mathbf{z}^o$. We define the residual component $\mathbf{z}^r$ as the variation introduced by IWTD. To isolate this residual component, an intuitive method is to subtract $\mathbf{z}^d$ and $\mathbf{z}^o$. For more flexible control over the residual component, we consider the following linear relationship:
\begin{equation} \label{eq:residual_def}
\mathbf{z}^d = \lambda \mathbf{z}^o + (1 - \lambda) \mathbf{z}^r,
\end{equation}
where $\mathbf{z}^r = \frac{\mathbf{z}^d - \lambda \mathbf{z}^o}{1 - \lambda}$ is the residual component and $\lambda \in (0, 1)$. 
We employ an annealing strategy~\cite{jing2023order}:
\scalebox{0.9}{$\lambda = \lambda_0 [ 1 - \left(1 + 10t/T\right)^{3/4}]$},
where $t$, $T$ are current and total iterations. This gradually increases $\lambda$ to balance original and residual components.
Based on Occam's razor principle, linearity is a good inductive bias~\cite{zhang2018mixup,jing2023order}, and Eq.~\eqref{eq:residual_def} is an invertible operation that can easily infer $\mathbf{z}^r$ given $\mathbf{z}^d$ and $\mathbf{z}^o$ (linear vs. non-linear analysis in Appendix C). Besides, assuming IWTD primarily targets on non-critical information, the residual component $\mathbf{z}^r$ should ideally contain minimal class-discriminative features. To encourage this, we aim to maximize the uncertainty associated with the class prediction based on $\mathbf{z}^r$.

Taking the visual modality as an example, we compute the cosine similarity $\text{sim}(\cdot)$, between the visual residual component $\mathbf{z}_v^r$ and the original textual class embeddings $\mathcal{W}_t^o = \{\mathbf{z}_{t,k}^o\}_{k=1}^K$. $\mathbf{z}_{t,k}^o$ are obtained by feeding text descriptions corresponding to each class $k$ into the VLM's text encoder $\mathcal{G}$. The probability distribution over classes given the visual residual is:
\begin{equation}
    p(y=k | \mathbf{z}_v^r) = \frac{\exp(\text{sim}(\mathbf{z}_v^r, \mathbf{z}_{t,k}^o) /\tau) }{\sum_{j=1}^K \exp(\text{sim}(\mathbf{z}_v^r, \mathbf{z}_{t,j}^o) /\tau)}.
\end{equation}

Next, we maximize the conditional entropy $\mathcal{H}(p(y | \mathbf{z}_v^r))$ to enlarge the uncertainty of $\mathbf{z}_v^r$'s prediction. Therefore, our objective for the visual modality is as follows:
\begin{equation} \label{entropy_re}
\hspace{-4pt}
    \mathcal{L}_{RE}^v{=}-\mathcal{H}(p(y|\mathbf{z}_v^r)){=}\sum_{k=1}^K p(y{=}k|\mathbf{z}_v^r) \log p(y{=}k|\mathbf{z}_v^r).
\end{equation}

By minimizing $\mathcal{L}_{RE}^v$, we regularize $\mathbf{z}_v^r$ to have an approximately equal probability of being associated with any category, thereby ensuring $\mathbf{z}_v^r$ does not contain class-discriminative information.

An analogous procedure is applied to the textual modality, yielding a residual entropy loss $\mathcal{L}_{RE}^t$. The total residual entropy regularization: $\mathcal{L}_{RE} = \mathcal{L}_{RE}^v + \mathcal{L}_{RE}^t$. This constrains IWTD by ensuring dropout-altered information does not carry significant class-specific signals, maintaining semantic alignment for knowledge transfer while allowing beneficial robustness variations.
During inference, class probabilities are computed using the learned prompts without dropout:
\begin{equation}
\hspace{-4pt}
p(y{=}k|x){=}\frac{\exp(\text{sim}(\mathcal{F}(V_{input}\!),\!\mathcal{G}(T_{input,k}))/\tau)}{\!\!\sum_{j{=}1}^K\!\exp(\text{sim}(\mathcal{F}(V_{input}\!),\!\mathcal{G}(T_{input,j}))\!/\!\tau)}.
\end{equation}

The learned representations from our dropout-based training directly contribute to robust inference.
\section{Experiments}\label{experiments}
In this section, we conduct extensive experiments on widely-used benchmarks to evaluate our proposed method. 
We assess its performance in base-to-novel generalization, cross-dataset evaluation (Appendix B.1), few-shot classification, and out-of-distribution generalization, comparing it against competitive vision-language prompt learning baselines.

\begin{table*}[!t]
  \centering
  \caption{\textbf{Base-to-novel generalization.} Comparison with CoOp and the methods mainly focusing on regularization techniques to improve generalization across 11 image recognition datasets. Bold values indicate the best results. HM: Harmonic Mean.}
  \vspace{-2pt}
  \setlength{\tabcolsep}{7.5pt}
  \renewcommand{\arraystretch}{0.88}
  \scalebox{0.92}{
    \begin{tabular}{lccc|ccc|ccc|ccc}
    \toprule
    \multirow{2}[4]{*}{Method} & \multicolumn{3}{c}{Average} & \multicolumn{3}{c}{ImageNet} & \multicolumn{3}{c}{Caltech101} & \multicolumn{3}{c}{OxfordPets} \\
\cmidrule{2-13}          & Base  & Novel & HM    & Base  & Novel & HM    & Base  & Novel & HM    & Base  & Novel & HM \\
    \midrule
    CoOp$_{\text{(IJCV'22)}}$ & 82.69  & 63.22  & 71.66  & 76.47  & 67.88  & 71.92  & 96.00  & 89.81  & 93.73  & 93.67  & 95.29  & 94.47  \\
    KgCoOp$_{\text{(CVPR'23)}}$ & 80.73  & 73.60  & 77.00  & 75.83  & 69.96  & 72.78  & 97.72  & 94.39  & 96.03  & 94.65  & 97.76  & 96.18  \\
    PSRC$_{\text{(ICCV'23)}}$ & 84.26  & 76.10  & 79.97  & 77.60  & 70.73  & 74.01  & 98.10  & 94.03  & 96.02  & 95.33  & 97.30  & 96.30  \\
    DeKgTCP$_{\text{(ICLR'25)}}$ & 84.96  & 76.38  & 80.44  & 77.40  & 69.20  & 73.07  & 98.64  & 95.20  & 96.89  & 94.47  & 97.76  & 96.09  \\
    TAP$_{\text{(ICLR'25)}}$ & 84.75  & 77.63  & 81.04  & 77.97  & 70.40  & 73.99  & \textbf{98.90}  & 95.50  & 97.17  & 95.80  & 97.73  & 96.76  \\
    TAC$_{\text{(CVPR'25)}}$ & 85.24  & 77.60  & 81.24  & \textbf{78.57}  & 71.03  & 74.61  & 98.57  & 95.27  & 96.89  & 95.93  & \textbf{98.17}  & 97.04  \\
    \midrule
    \rowcolor[HTML]{e9f0fb}
    DroPLe$_{\text{(Ours)}}$ & \textbf{86.12}  & \textbf{78.44}  & \textbf{82.10}  & 78.24  & \textbf{71.38}  & \textbf{74.65}  & 98.72  & \textbf{96.06}  & \textbf{97.37}  & \textbf{96.38}  & 98.13  & \textbf{97.25}  \\
    \midrule
    \multirow{2}[4]{*}{Method} & \multicolumn{3}{c}{StanfordCars} & \multicolumn{3}{c}{Flowers102} & \multicolumn{3}{c}{Food101} & \multicolumn{3}{c}{FGVCAircraft} \\
\cmidrule{2-13}          & Base  & Novel & HM    & Base  & Novel & HM    & Base  & Novel & HM    & Base  & Novel & HM \\
    \midrule
    CoOp$_{\text{(IJCV'22)}}$ & 78.12  & 60.40  & 68.13  & 97.60  & 59.67  & 74.06  & 88.33  & 82.26  & 85.19  & 40.44  & 22.30  & 28.75  \\
    KgCoOp$_{\text{(CVPR'23)}}$ & 71.76  & 75.04  & 73.36  & 95.00  & 74.73  & 83.65  & 90.50  & 91.70  & 91.09  & 36.21  & 33.55  & 34.83  \\
    PSRC$_{\text{(ICCV'23)}}$ & 78.27  & 74.97  & 76.58  & 98.07  & 76.50  & 85.95  & 90.67  & 91.53  & 91.10  & 42.73  & 37.87  & 40.15  \\
    DeKgTCP$_{\text{(ICLR'25)}}$ & 81.18  & 74.75  & 77.83  & 98.58  & 75.18  & 85.30  & 90.73  & 91.55  & 91.14  & 45.20  & 35.09  & 39.51  \\
    TAP$_{\text{(ICLR'25)}}$ & 80.70  & 74.27  & 77.35  & 97.90  & 75.57  & 85.30  & 90.97  & 91.83  & 91.40  & 44.40  & 36.50  & 40.06  \\
    TAC$_{\text{(CVPR'25)}}$ & 81.63  & 74.17  & 77.72  & 97.97  & 76.87  & 86.15  & 90.87  & 91.87  & 91.37  & 44.60  & 37.70  & 40.86  \\
    \midrule
    \rowcolor[HTML]{e9f0fb}
    DroPLe$_{\text{(Ours)}}$ & \textbf{82.87}  & \textbf{75.04}  & \textbf{78.76}  & \textbf{98.61}  & \textbf{78.29}  & \textbf{87.28} & \textbf{91.18}  & \textbf{92.20}  & \textbf{91.69}  & \textbf{49.26}  & \textbf{39.35}  & \textbf{43.75}  \\
    \midrule
    \multirow{2}[4]{*}{Method} & \multicolumn{3}{c}{SUN397} & \multicolumn{3}{c}{DTD} & \multicolumn{3}{c}{EuroSAT} & \multicolumn{3}{c}{UCF101} \\
\cmidrule{2-13}          & Base  & Novel & HM    & Base  & Novel & HM    & Base  & Novel & HM    & Base  & Novel & HM \\
    \midrule
    CoOp$_{\text{(IJCV'22)}}$ & 80.60  & 65.89  & 72.51  & 79.44  & 41.18  & 54.24  & 93.19  & 54.74  & 68.69  & 84.69  & 56.05  & 67.46  \\
    KgCoOp$_{\text{(CVPR'23)}}$ & 80.29  & 76.53  & 78.36  & 77.55  & 54.99  & 64.35  & 85.64  & 64.34  & 73.48  & 82.89  & 76.67  & 79.65  \\
    PSRC$_{\text{(ICCV'23)}}$ & 82.67  & 78.47  & 80.52  & 83.37  & 62.97  & 71.75  & 92.90  & 73.90  & 82.32  & 87.10  & 78.80  & 82.74  \\
    DeKgTCP$_{\text{(ICLR'25)}}$ & 82.52  & 78.30  & 80.35  & 83.80  & 59.66  & 69.70  & 94.02  & 81.69  & 87.42  & 88.06  & 81.77  & 84.80  \\
    TAP$_{\text{(ICLR'25)}}$ & 82.87  & 79.53  & 81.17  & 84.20  & \textbf{68.00}  & 75.24  & 90.70  & 82.17  & 86.22  & 87.90  & 82.43  & 85.08  \\
    TAC$_{\text{(CVPR'25)}}$ & 83.70  & 80.03  & 81.82  & 83.37  & 64.27  & 72.58  & 94.37  & 82.60  & 88.10  & 88.07  & 81.67  & 84.75  \\
    \midrule
    \rowcolor[HTML]{e9f0fb}
    DroPLe$_{\text{(Ours)}}$ & \textbf{83.82}  & \textbf{80.07}  & \textbf{81.90}  & \textbf{85.43}  & 67.32  & \textbf{75.30}  & \textbf{94.73}  & \textbf{82.70}  & \textbf{88.31}  & \textbf{88.16}  & \textbf{82.73}  & \textbf{85.39}  \\
    \bottomrule
    \end{tabular}%
    }
  \label{tab_b2n}%
  \vspace{-5pt}
\end{table*}%
\paragraph{Datasets.}
Following the previous work \cite{zhou2022learning}, our experiments utilize a diverse array of 11 image classification datasets: UCF101 \cite{soomro2012ucf101} (action recognition), DTD \cite{cimpoi2014describing} (texture analysis), SUN397 \cite{xiao2016sun} (scene recognition), EuroSAT \cite{helber2019eurosat} (satellite imagery), five fine-grained datasets (Flowers102 \cite{nilsback2008automated}, FGVCAircraft \cite{maji2013fine}, Food101 \cite{bossard2014food}, OxfordPets \cite{parkhi2012cats}, StanfordCars \cite{krause20133d}), and two generic object datasets (Caltech101 \cite{fei2004learning}, ImageNet \cite{deng2009imagenet}). For evaluating out-of-distribution generalization, ImageNet serves as the source dataset, while its variants (ImageNet-A \cite{hendrycks2021natural}, ImageNet-R \cite{hendrycks2021many}, ImageNet-Sketch \cite{wang2019learning}, ImageNet-V2 \cite{recht2019imagenet}) are target datasets.

\paragraph{Implementation Details.}
Our implementation is based on CLIP-B/16~\cite{radford2021learning}. 
The number of bridge tokens $\xi$ is set to 64. The textual and visual prompt learning layers are set to 9 and 6, with corresponding adaptive dropout layers are 6. Learnable prompt length is 4.
The textual token space
is augmented with 16 learnable tokens for expanded dropout operation, with $T^{(i)}_{\text{input}} = \{t_{\text{bos}}, P^{(i)}_t, T_{\text{embed}}, t_{\text{eos}}\}$, where $P^{(i)}_t$ denotes supplementary tokens.
The default dropout probability ranges from 10\% to 50\%.
$\lambda_0$ is set to 0.1. Experiments are conducted on 2 NVIDIA 4090 GPUs.

\subsection{Base-to-Novel Generalization}\label{b2n}
In the base-to-novel generalization, we evaluate our dropout prompt learning method, dubbed as \textbf{DroPLe}, against recent approaches specifically designed to enhance prompt learning generalization via regularization techniques
(KgCoOp~\cite{yao2023visual}, PSRC~\cite{khattak2023self}, DeKgTCP~\cite{li2025divergence}), the baseline CoOp~\cite{zhou2022learning}, the LLM-enhanced method TAP~\cite{ding2024tree} and TAC~\cite{hao2025task} which incorporates textual-visual consistency regularization. As shown in Table~\ref{tab_b2n}, with 16-shot training on base classes and zero-shot evaluation on novel classes, DroPLe achieves 86.12\% base accuracy and 78.44\% novel accuracy, with 82.10\% HM significantly surpassing TAP (81.04\%) and DeKgTCP (80.44\%). Notably, DroPLe achieves the best overall HM across all 11 datasets and demonstrates superior performance on fine-grained recognition, achieving 43.75\% HM on FGVCAircraft (+2.89\% over TAC). Furthermore, to verify DroPLe's generalization in long-tailed scenarios, we follow Candles~\cite{shi2024efficient}, replacing loss functions in CoOp~\cite{zhou2022learning} and CoCoOp~\cite{zhou2022conditional} with LA Loss~\cite{ren2020balanced}. As shown in Fig.~\ref{fig_imbalance} (a), under an imbalance ratio of 10, DroPLe outperforms LFA~\cite{ouali2023black} and other long-tail classification methods in terms of harmonic mean accuracy. On EuroSAT, DroPLe achieves +4.6\% gain over the 75.6\% baseline, outperforming GLA's +3.8\%~\cite{zhu2023generalized}. Notably, DroPLe is orthogonal to such post-hoc methods, with GLA+DroPLe yielding +5.4\% improvement. This further demonstrates its effectiveness and robustness in handling imbalanced data distributions.
\begin{figure}
\centering
\includegraphics[width=1\linewidth]{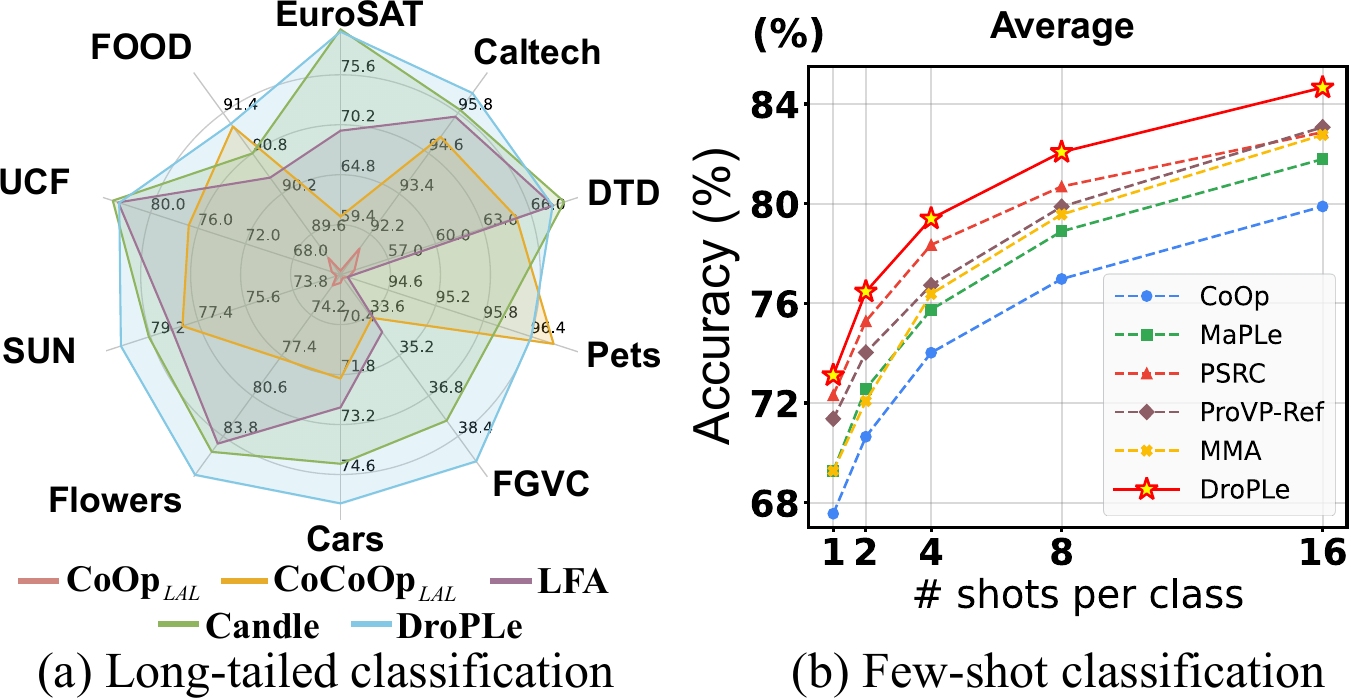} 
  \vspace{-6pt}
  \caption{(a) Base-to-novel task with \textbf{imbalance ratio 10}. (b) \textbf{Few-shot classification.} We conduct experiments on 11 datasets. Per-dataset results are in the Appendix B.2.}
  \label{fig_imbalance}
  \vspace{-5pt}
\end{figure}

\begin{table}
  \centering
  \vspace{-3pt}
  \caption{\textbf{Out-of-distribution generalization.} `*'~means reproduced results. Best results highlighted in \colorbox[HTML]{d9e6fa}{\textbf{first}}, \colorbox[HTML]{ebf0fa}{second}.}
  \setlength{\tabcolsep}{6.5pt}
  \renewcommand{\arraystretch}{0.90}
  \scalebox{0.84}{
    \begin{tabular}{lc|cccc|>{\columncolor[HTML]{f8f8f8}}c}
    \toprule
    \multirow{2}[4]{*}{\textbf{Method}} & \textbf{Source} & \multicolumn{5}{c}{\textbf{Target}} \\
\cmidrule{2-7}          & \textbf{ImgNet} & \textbf{-V2}  & \textbf{-S}    & \textbf{-A}    & \textbf{-R}    & \textbf{\textit{OOD}} \\
    \midrule
    KgCoOp & 71.20  & 64.10  & 48.97  & 50.69  & 76.70  & 60.11  \\
    PSRC  & 71.27  & 64.35  & 49.55  & 50.90  & 77.80  & 60.65  \\
    CoPrompt & 70.80  & 64.25  & 49.43  & 50.50  & 77.51  & 60.42  \\
    ProMetaR & \cellcolor[HTML]{ebf0fa}71.29  & 64.39  & 49.55  & 51.25  & \cellcolor[HTML]{ebf0fa}77.89  & 60.77  \\
    GalLoP* & 71.14  & 64.32  & \cellcolor[HTML]{ebf0fa}49.56  & 50.83  & 77.42  & 60.53  \\
    SPTR & 70.05  & \cellcolor[HTML]{ebf0fa}64.40  & 48.78  & \cellcolor[HTML]{ebf0fa}51.30  & \cellcolor[HTML]{d9e6fa}\textbf{77.90}  & \cellcolor[HTML]{ebf0fa}60.59  \\
    \midrule
    DroPLe  & \cellcolor[HTML]{d9e6fa}\textbf{71.94}  & \cellcolor[HTML]{d9e6fa}\textbf{64.90}  & \cellcolor[HTML]{d9e6fa}\textbf{50.36}  & \cellcolor[HTML]{d9e6fa}\textbf{51.33}  & 77.59  & \cellcolor[HTML]{d9e6fa}\textbf{61.05}  \\
    \bottomrule
    \end{tabular}%
   }
  \label{tab:ood}%
  \vspace{-5pt}
\end{table}


\subsection{Few-shot Classification}\label{few-shot classification_appendix}
Fig.~\ref{fig_imbalance}(b) presents the comprehensive few-shot classification results averaged across all 11 datasets. DroPLe demonstrates strong few-shot learning capabilities, showing competitive performance across different shot settings. The method achieves favorable results with our importance weighted token dropout providing effective regularization in data-limited scenarios. Detailed per-dataset results are provided in Appendix B.2, where DroPLe achieves notable performance on challenging datasets such as 97.41\% on Caltech101 and 94.57\% on OxfordPets at 16 shots per class. The results validate our adaptive dropout strategy successfully maintains crucial semantic information while introducing beneficial diversity for improved generalization in few-shot learning.

\subsection{Out-of-distribution Generalization}\label{o.o.d.}
We further evaluate DroPLe's robustness on out-of-distribution (OOD) generalization. As shown in Table~\ref{tab:ood}, DroPLe attains the highest average OOD accuracy of 61.05\%, surpassing SPTR~\cite{cui2025similarity} by 0.46\%. Notably, compared with GalLoP~\cite{lafon2024gallop} that applies text-level dropout, our token-level adaptive dropout strategy achieves better performance on ImageNet-S (+0.80\%) and ImageNet-V2 (+0.58\%). This indicates that token-level dropout enables more precise regularization than text-level dropout, leading to better generalization across distribution shifts.

\begin{table}
  \centering
  \caption{\textbf{Ablation studies of components} on base-to-novel task. $\mathcal{D}_{cos}$ is cosine similarity.
  $\mathcal{D}_{KL}$ is KL divergence. DroPLe$^{\blacktriangle}$ only uses the score $S_{self}$, DroPLe$^{\blacklozenge}$ adds $S_{cls}$ on $^{\blacktriangle}$, and DroPLe$^{\bigstar}$ further adds $S_{cross}$ on $^{\bigstar}$.
  } 
  \renewcommand{\arraystretch}{0.95}
  \newcolumntype{C}[1]{>{\centering\arraybackslash}p{#1}}
  \scalebox{0.76}{
    \begin{tabular}{C{0.7cm}C{0.7cm}|p{1.5cm}p{2.7cm}|C{0.7cm}C{0.7cm}C{0.7cm}}
    \toprule
    IWTD  & $\mathcal{L}_{RE}$ & Component & \multicolumn{1}{l|}{Method} & Base  & Novel & HM \\
    \midrule
    \multirow{4}[2]{*}{\ding{55}} & \multirow{4}[2]{*}{\ding{55}} &       & Baseline & 83.82 & 75.50  & 79.44  \\ 
           &       & \multirow{3}[1]{*}{\textit{Dropout}} & \ +Dropout$_{0.5}$ & 82.71 & 73.45 & 77.81  \\
          
          &      &       & \ +Dropout$_{0.3}$ & 83.44 & 74.53 & 78.73  \\
          &     &       & \ +Dropblock & 82.63 & 73.75 & 77.94  \\
          \cmidrule{1-7}
        $\checkmark$  &   \ding{55}     &       & \ +IWTD (Ours) & 85.34 & 77.38 & 81.17  \\
\cmidrule{1-7}    \multirow{2}[0]{*}{$\checkmark$} & \multirow{2}[0]{*}{\ding{55}} & \multirow{2}[0]{*}{\parbox{1.5cm}{\textit{Consistency Reg.}}} & $\mathcal{D}_{cos}$  & 85.53 & 77.60 & 81.37  \\
          &       &       & $\mathcal{D}_{KL}$+$\mathcal{L}_{1}$ & 85.80 & 77.53 & 81.46  \\
          \cmidrule{1-7} 
        $\checkmark$  &    $\checkmark$   &       & $\mathcal{L}_{RE}$ (Ours) & 86.12 & 78.44 & 82.10  \\
\cmidrule{1-7}    \multirow{3}[2]{*}{$\checkmark$} & \multirow{3}[2]{*}{$\checkmark$} & \multirow{3}[2]{*}{\parbox{1.5cm}{\textit{Importance\\Scores}}} & \cellcolor[HTML]{f3f5fa}DroPLe$^{\blacktriangle}_{S_{self}}$ & \cellcolor[HTML]{ebf0fa}84.78 & \cellcolor[HTML]{ebf0fa}76.92 & \cellcolor[HTML]{ebf0fa}80.66  \\
          &       &       & \cellcolor[HTML]{e6edfa}DroPLe$^{\blacklozenge}_{S_{self,cls}}$ & \cellcolor[HTML]{e6edfa}85.24 & \cellcolor[HTML]{e6edfa}77.41 & \cellcolor[HTML]{e6edfa}81.14  \\
          &       &       & \cellcolor[HTML]{d2e2fa}DroPLe$^{\bigstar}_{S_{self,cls,cross}}$ & \cellcolor[HTML]{d2e2fa}\textbf{86.12} & \cellcolor[HTML]{d2e2fa}\textbf{78.44} & \cellcolor[HTML]{d2e2fa}\textbf{82.10}  \\
    \bottomrule
    \end{tabular}%
  }
  \label{tab:ablation}%
  \vspace{-3pt}
\end{table}

\subsection{Ablation study}\label{ablation_study}
\noindent\textbf{Component Ablation.}
Table~\ref{tab:ablation} shows the ablation analysis on base-to-novel task. The baseline implements a similar architecture to Fig.~\ref{pipeline} but removes IWTD and $\mathcal{L}_{RE}$, using $\mathcal{L}_2$ norm to constrain dual-branch embeddings as in \cite{yao2023visual}. While vanilla dropout (randomly dropping individual tokens, dropout probability is 50\% or 30\%) and dropblock (dropping consecutive 3 tokens) show limited gains, IWTD achieves substantial improvement (81.17\% HM) by preserving semantically meaningful tokens.
Upon IWTD, we examine consistency regularization approaches. Although CoPrompt's $\mathcal{D}_{cos}$ and PSRC's $\mathcal{D}_{KL}$+$\mathcal{L}_{1}$ improve performance, our proposed $\mathcal{L}_{RE}$ further boosts HM to 82.10\% by encouraging representation diversity. Finally, ablating scores in Eq.~(\ref{eq:importance_metric}) shows the cross-modal score $S_{cross}$ contributes the largest gain, validating its importance for VLM alignment.

\noindent\textbf{Configuration Analysis.}
Results of hyperparameter and baseline analysis are shown in Table~\ref{tab:diff_exper_studies}.
Vanilla dropout typically employs a 50\% dropout probability, as higher rates risk eliminating critical features. Our method performs targeted dropout based on token importance, maintaining robustness even with higher probability thresholds, 70\%. Empirically, the 10-50\% range yields optimal results and serves as our default configuration.
The number of shared tokens $\xi$ is set to 64, balancing cross-modal semantic representation with computational cost.
For the initial value $\lambda_0$ of $\lambda$ in Eq.~(\ref{eq:residual_def}), we set $\lambda_0=0.1$ to emphasize IWTD residual features during early training and enhance model robustness.
As demonstrated in Table~\ref{tab:diff_exper_studies}(d), DroPLe consistently improves generalization on challenging OOD tasks across various prompt learning baselines, confirming our method's effectiveness.

\begin{table}[htbp]
  \centering
  \caption{\textbf{Ablation analysis of different settings.} The default configuration is colored \colorbox[HTML]{e9f0fb}{blue}.}
  \label{tab:diff_exper_studies}
  \hspace{-4pt}
  \begin{minipage}{0.24\textwidth}
    \small
    (a) Range of \textbf{dropout probability}.
    \setlength{\tabcolsep}{4.5pt}
    \newcolumntype{C}[1]{>{\centering\arraybackslash}p{#1}}
    \scalebox{0.9}{
    \hspace{-6pt}
      \begin{tabular}{cccc}
      Range (\%)  & Base  & Novel & HM \\
      \midrule
      0-40  & 86.08  & 78.21  & 81.96  \\
      5-45 & 86.04  & 78.35  & 82.02  \\
      \rowcolor[HTML]{e9f0fb}
      10-50  & 86.12  & 78.44  & 82.10  \\
      20-70 & 85.85  & 78.13  & 81.81  \\
      \end{tabular}%
    }
  \end{minipage}%
  \begin{minipage}{0.24\textwidth}
    \hspace{2pt}
    \centering
    \small
    (b) Number of \textbf{shared tokens} $\xi$.
    \setlength{\tabcolsep}{5.5pt}
    \scalebox{0.9}{
      \begin{tabular}{cccc}
      $\xi$ & Base  & Novel & HM \\
      \midrule
      16    & 85.74  & 78.39  & 81.90  \\
      32    & 85.97  & 78.36  & 81.99  \\
      \rowcolor[HTML]{e9f0fb}
      64    & 86.12  & 78.44  & 82.10  \\
      128   & 86.20  & 78.41  & 82.12  \\
      \end{tabular}%
    }
  \end{minipage}

  \vspace{0.5em}

  \hspace{-24pt}
  \begin{minipage}{0.21\textwidth}
    \centering
    \small
    (c) Initial \textbf{balance para-\\meter} $\lambda_0$.
    \setlength{\tabcolsep}{3pt}
    \scalebox{0.9}{
      \begin{tabular}{cccc}
      $\lambda_0$ & Base  & Novel & HM \\
      \midrule
      0.05  & 86.02  & 78.48  & 82.08  \\
      \rowcolor[HTML]{e9f0fb}
      0.10  & 86.12  & 78.44  & 82.10  \\
      0.20  & 85.97  & 78.28  & 81.94  \\
      0.40  & 85.93  & 78.26  & 81.92  \\
      \end{tabular}%
    }
  \end{minipage}%
  \hspace{0.5pt}
  \begin{minipage}{0.25\textwidth}
    \centering
    \small
    \hspace{20pt}
    (d) OOD~generalization~of \textbf{other approaches}.
    \setlength{\tabcolsep}{2.5pt}
    \scalebox{0.78}{
      \begin{tabular}{l|ccccc}
      Method & ImgNet & -V2   & -S    & -A    & -R \\
      \midrule
      CoOp  & 71.51  & 64.20  & 47.99  & 49.71  & 75.21  \\
      \rowcolor[HTML]{e9f0fb}
      \ \ +DroPLe & 71.32  & 64.23  & 49.06  & 50.93  & 77.04  \\
      \midrule
      MaPLe & 70.72  & 64.07  & 49.15  & 50.90  & 76.98  \\
      \rowcolor[HTML]{e9f0fb}
      \ \ +DroPLe & 71.26  & 64.38  & 49.56  & 51.28  & 77.84  \\
      \end{tabular}%
    }
  \end{minipage}
\end{table}

\noindent\textbf{Cross-Architecture Evaluation.}
We evaluate our method across diverse VLM architectures and adapter-based methods, demonstrating broad effectiveness (see Appendix B.3).

\noindent\textbf{Visualization.}
Grad-CAM visualizations reveal that our importance weighted token dropout produces concentrated attention on semantically relevant regions, while vanilla dropout exhibits scattered patterns (Fig. 5(a)), showing our method effectively maintains cross-modal semantic alignment.

\begin{figure}[t]
    \centering    \includegraphics[width=1.0\linewidth]{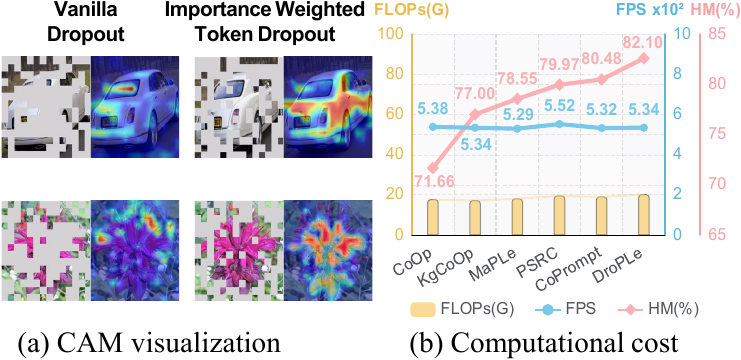}
    \vspace{-12pt}
    \caption{(a) \textbf{Grad-CAM visualizations for different dropout methods.} Redder colors indicate higher feature attention. (b) \textbf{Computational cost and performance} of different prompt learning methods.}
    \vspace{-12pt}
    \label{cam_dropout}
\end{figure}

\noindent\textbf{Computational Efficiency.}
Our method maintains comparable FLOPs and inference speed while achieving superior performance (HM: 82.10\%) over CoPrompt (80.48\%) and PSRC (79.97\%), demonstrating efficient enhancement without computational overhead, as shown in Fig.~\ref{cam_dropout}(b).
\section{Conclusion}\label{conclusion}
This paper introduces Dropout Prompt Learning, a novel paradigm that enhances vision-language model generalization through token-level dropout. We propose importance weighted token dropout, which dynamically assesses token significance by considering intra-modal context, inter-modal alignment, and task-specific relevance. Additionally, we introduce residual entropy regularization to maintain general knowledge and promote representational diversity. Extensive experiments across diverse benchmarks demonstrate our method's effectiveness in low-shot learning, long-tail classification, and out-of-distribution generalization, showing promise for robust VLM adaptation.

\bibliography{aaai2026}
\clearpage

\clearpage
\appendix
\setcounter{page}{1}
\renewcommand{\thesection}{\Alph{section}}
\setcounter{section}{0}
\renewcommand{\thetable}{a\arabic{table}}
\setcounter{table}{0}
\renewcommand{\thefigure}{a\arabic{figure}}
\setcounter{figure}{0}
\renewcommand{\theequation}{a.\arabic{equation}}
\setcounter{equation}{0}

\twocolumn[
\begin{@twocolumnfalse}
\centerline{{\LARGE \textbf{Appendix for ``Dropout Prompt Learning:}}}
\centerline{{\LARGE \textbf{Towards Robust and Adaptive Vision-Language Models''}}}
\vspace{1em}
\end{@twocolumnfalse}
]

\setcounter{proposition}{0}
\section{A. Theoretical Analysis} %
\subsection{A.1. Generalization Bounds for IWTD}
\begin{proposition}
\label{prop:iwtd_rademacher_comparison}
Given a training sample $\mathcal{X}$ of $n$ instances, let $\mathbb{F}_{\mathbf{q}}$ be the hypothesis space induced by layer-wise token retention probabilities $\mathbf{q}$, and $h \in \mathbb{F}_{\mathbf{q}}$ be a learned hypothesis. The expected risk $R(h)$ is bounded, with probability at least $1-\delta$, by:
\begin{align}
R(h) \leq \hat{R}_{\mathcal{X}}(h) + 2 C_{\ell} \hat{\mathcal{R}}_{\mathcal{X}}(\mathbb{F}_{\mathbf{q}}) + B_{loss} \sqrt{\frac{\ln(1/\delta)}{2n}},
\label{eq:std_gen_bound}
\end{align}
where $\hat{R}_{\mathcal{X}}(h)$ denotes the empirical risk, $\ell$ represents the $C_{\ell}$-Lipschitz loss function bounded by $B_{loss}$, $\hat{\mathcal{R}}_{\mathcal{X}}(\mathbb{F}_{\mathbf{q}})$ measures the empirical Rademacher complexity, and $\delta \in (0,1)$ specifies the confidence level. The deviation of expected risk $R(h)$ from empirical risk $\hat{R}_{\mathcal{X}}(h)$ is predominantly determined by $\hat{\mathcal{R}}_{\mathcal{X}}(\mathbb{F}_{\mathbf{q}})$, where a smaller complexity yields a more precise bound.

Let $\mathbb{F}_{\text{IWTD}}$ and $\mathbb{F}_{\text{VTD}}$ denote the hypothesis spaces induced by IWTD with adaptive retention probabilities $\mathbf{q}_{\text{IWTD}}$ and vanilla token dropout (VTD) with uniform retention probability $q_v$, respectively. When VTD achieves comparable empirical performance with IWTD, we have:
\begin{align}
\hat{\mathcal{R}}_{\mathcal{X}}(\mathbb{F}_{\text{IWTD}}) \leq \hat{\mathcal{R}}_{\mathcal{X}}(\mathbb{F}_{\text{VTD}}). 
\end{align}
\end{proposition}

\begin{proof}
The generalization bound in Eq.~\eqref{eq:std_gen_bound} is a standard result relying on Rademacher complexity \cite{bartlett2002rademacher, mohri2018foundations}. Our focus is the comparison between $\hat{\mathcal{R}}_{\mathcal{X}}(\mathbb{F}_{\text{IWTD}})$ and $\hat{\mathcal{R}}_{\mathcal{X}}(\mathbb{F}_{\text{VTD}})$.

Theoretical analyses show that the Rademacher complexity of networks with dropout depends on factors including network weights and the dropout or retention probabilities, often exhibiting a multiplicative structure across layers \cite{gao2016dropout, zhai2018adaptive}. This can be conceptually represented as:
\begin{align}
\hat{\mathcal{R}}_{\mathcal{X}}(\mathbb{F}_{\mathbf{q}}) \leq \text{BaseTerm}(\mathcal{X}, \mathbf{W}) \cdot \prod_{i=1}^{K_L} \Psi(\mathbf{W}^{(i)}, \mathbf{q}^{(i)}), \label{eq:rad_prod_bound_proof_final_appendix}
\end{align}
where $K_L$ is the number of layers, $\text{BaseTerm}(\mathcal{X}, \mathbf{W})$ captures dependencies on the sample and overall weights, and $\Psi(\mathbf{W}^{(i)}, \mathbf{q}^{(i)})$ reflects the contribution of layer $i$, depending on its weights $\mathbf{W}^{(i)}$ and retention probability vector $\mathbf{q}^{(i)}$. Bounds like those in \cite{zhai2018adaptive} indicate a dependency on terms related to the $L_1$ norm of $\mathbf{q}^{(i)}$. Therefore, to compare the effect of the dropout strategies while assuming the weight-dependent factors are comparable, we analyze the factor $\Phi_i(\mathbf{q}^{(i)}) = ||\mathbf{q}^{(i)}||_1 = \sum_{j=1}^{N_i} q_j^{(i)}$, where $q_j^{(i)}$ is the retention probability for token $j$ at layer $i$ (with $N_i$ tokens).

\textbf{VTD:} For Vanilla Token Dropout, $q_j^{(i)} = q_v$ for all $j$. The corresponding factor is:
\begin{align}
\Phi_i(\mathbf{q}_{\text{VTD}}^{(i)}) = \sum_{j=1}^{N_i} q_v = N_i q_v. \label{eq:vtd_phi_proof_l1_appendix}
\end{align}

\textbf{IWTD:} For Importance Weighted Token Dropout, $q_j^{(i)} = (1-p_{max}) + \hat{I}(\mathbf{x}^{(i)}_j)(p_{max} - p_{min})$, where $\hat{I}(\mathbf{x}^{(i)}_j) \in [0,1]$ is the normalized importance. Probabilities $q_j^{(i)}$ fall within $[q_{low}, q_{high}]$, where $q_{low}=1-p_{max}$ and $q_{high}=1-p_{min}$. Let $N_{imp}$ be the number of important tokens and $N_{unimp}$ be the number of unimportant tokens, with $N_i = N_{imp} + N_{unimp}$. The factor is:
\begin{equation}
\begin{aligned}
\Phi_i(\mathbf{q}_{\text{IWTD}}^{(i)}) &= \sum_{j \text{ s.t. } q_j^{(i)} \text{ is high}} q_j^{(i)} + \sum_{j \text{ s.t. } q_j^{(i)} \text{ is low}} q_j^{(i)}
\\
&\leq N_{imp} q_{high} + N_{unimp} q_{low}, 
\end{aligned}
\label{eq:iwtd_phi_proof_l1_appendix}
\end{equation}
where the inequality arises because individual $q_j^{(i)}$ are bounded by $q_{high}$ and $q_{low}$ respectively for tokens deemed important or unimportant by the IWTD mechanism.

\textbf{Comparison:} We compare $\Phi_i(\mathbf{q}_{\text{IWTD}}^{(i)})$ and $\Phi_i(\mathbf{q}_{\text{VTD}}^{(i)})$ assuming VTD uses a retention probability $q_v$ that balances regularization and performance preservation. For fair comparison, we consider the case where $q_v$ is set to maintain reasonable performance, which typically requires $q_v$ to be relatively high to avoid dropping too many informative tokens uniformly.
Under this assumption, we have $\Phi_i(\mathbf{q}_{\text{VTD}}^{(i)}) = N_i q_v$. Note that we assume $N_{unimp}>0$ and $q_{low} < q_{high}$, which are reasonable given that there exist unimportant dimensions and the voting threshold is higher than the low-quality threshold. Comparing this with Eq.~\eqref{eq:iwtd_phi_proof_l1_appendix}:
\begin{equation}
\begin{split}
\Phi_i(\mathbf{q}_{\text{IWTD}}^{(i)}) &\leq N_{imp} q_{high} + N_{unimp} q_{low} \\
&< N_{imp} q_{high} + N_{unimp} q_{high} \\
&= (N_{imp} + N_{unimp}) q_{high} \\
&= N_i q_{high} \\
&\approx N_i q_v = \Phi_i(\mathbf{q}_{\text{VTD}}^{(i)}).
\end{split}
\label{eq:your_label}
\end{equation}
Thus, under the stated premise and assuming $N_{unimp}>0$, $\Phi_i(\mathbf{q}_{\text{IWTD}}^{(i)}) < \Phi_i(\mathbf{q}_{\text{VTD}}^{(i)})$. If $N_{unimp}=0$ or if approximations are equalities, $\Phi_i(\mathbf{q}_{\text{IWTD}}^{(i)}) \leq \Phi_i(\mathbf{q}_{\text{VTD}}^{(i)})$. This suggests that IWTD achieves a sum of retention probabilities less than or equal to that of VTD while maintaining comparable performance on critical tokens. This aligns with the principle that adaptive regularization techniques can potentially lead to lower model complexity compared to non-adaptive counterparts for a similar level of empirical performance \cite{zhai2018adaptive}.

\begin{table*}[!t]
  \centering
  \caption{\textbf{Cross-dataset evaluation.} Comparison with CoOp and regularized prompt learning methods.  Trained on ImageNet (16 shots), evaluated on 10 datasets. Best results highlighted in \colorbox[HTML]{d9e6fa}{\textbf{first}}, \colorbox[HTML]{ebf0fa}{second}.}
  \vspace{-5pt}
  \renewcommand{\arraystretch}{0.95}
  \scalebox{1}{
    \begin{tabular}{lc|cccccccccc>{\columncolor[HTML]{f8f8f8}}c}
    \toprule
    \multirow{2}[4]{*}{} & \textbf{Source} & \multicolumn{11}{c}{\textbf{Target}} \\
\cmidrule{2-13} & \rotatebox{60}{\textbf{ImgNet}} & \rotatebox{60}{\textbf{Cal101}} & \rotatebox{60}{\textbf{Pets}}   & \rotatebox{60}{\textbf{Cars}}  & \rotatebox{60}{\textbf{Flowers}} & \rotatebox{60}{\textbf{Food}}  & \rotatebox{60}{\textbf{FGVC}}  & \rotatebox{60}{\textbf{SUN}}  & \rotatebox{60}{\textbf{DTD}}  & \rotatebox{60}{\textbf{SAT}}  & \rotatebox{60}{\textbf{UCF}}  & \rotatebox{60}{\textbf{\textit{Avg.}}} \\
    \midrule
    \midrule
    CoOp$_{\text{(IJCV'22)}}$ & 71.51  & 93.70  & 89.14  & 64.51  & 68.71  & 85.30  & 18.47  & 64.15  & 41.92  & 46.39  & 66.55  & 63.88  \\
    KgCoOp$_{\text{(CVPR'23)}}$ & 70.66  & 93.92  & 89.83  & 65.41  & 70.01  & 86.36  & 22.51  & 66.16  & 46.35  & 46.04  & 68.50  & 65.51  \\
    PSRC$_{\text{(ICCV'23)}}$ & 71.27  & 93.60  & 90.25  & 65.70  & 70.25  & 86.15  & 23.90  & 67.10  & 46.87  & 45.50  & 68.75  & 65.81  \\
    CoPrompt$_{\text{(ICLR'24)}}$ & 70.80  & 94.50  & \cellcolor[HTML]{ebf0fa}90.73  & 65.67  & 72.30  & \cellcolor[HTML]{ebf0fa}86.43  & 24.00  & \cellcolor[HTML]{ebf0fa}67.57  & 47.07  & \cellcolor[HTML]{ebf0fa}51.90  & \cellcolor[HTML]{d9e6fa}\textbf{69.73} & \cellcolor[HTML]{ebf0fa}67.00  \\
    ProMetaR$_{\text{(CVPR'24)}}$ & 71.29  & 93.74  & 90.59  & \cellcolor[HTML]{ebf0fa}65.83  & 71.13  & 86.39  & 24.78  & 67.41  & 47.08  & 45.02  & \cellcolor[HTML]{ebf0fa}69.50  & 66.15  \\
    DeKgTCP$_{\text{(ICLR'25)}}$ & \cellcolor[HTML]{d9e6fa}\textbf{72.33}  & \cellcolor[HTML]{ebf0fa}94.73  & 90.02  & 65.49  & \cellcolor[HTML]{d9e6fa}\textbf{72.39}  & \cellcolor[HTML]{d9e6fa}\textbf{86.59}  & \cellcolor[HTML]{ebf0fa}25.05  & 67.19  & 44.47  & 51.37  & 68.78  & 66.61  \\
    TAP$_{\text{(ICLR'25)}}$ & \cellcolor[HTML]{ebf0fa}72.30  & 94.30  & 90.70  & 65.60  & 70.93  & 86.10  & 24.57  & \cellcolor[HTML]{d9e6fa}\textbf{68.30}  & \cellcolor[HTML]{d9e6fa}\textbf{50.20}  & 46.00  & 68.90  & 66.56  \\
    \midrule
    DroPLe$_{\text{(Ours)}}$ & 71.94  & \cellcolor[HTML]{d9e6fa}\textbf{94.77}  & \cellcolor[HTML]{d9e6fa}\textbf{91.06}  & \cellcolor[HTML]{d9e6fa}\textbf{66.18}  & \cellcolor[HTML]{ebf0fa}72.37  & 86.17  & \cellcolor[HTML]{d9e6fa}\textbf{26.36}  & 67.33  & \cellcolor[HTML]{ebf0fa}47.45  & \cellcolor[HTML]{d9e6fa}\textbf{52.63}  & 68.46  & \cellcolor[HTML]{d9e6fa}\textbf{67.28}  \\
    \bottomrule
    \end{tabular}%
}
\vspace{-3pt}
  \label{tab:cross_dataset}%
\end{table*}%

\begin{figure*}[!t]
    \centering    \includegraphics[width=1.0\linewidth]{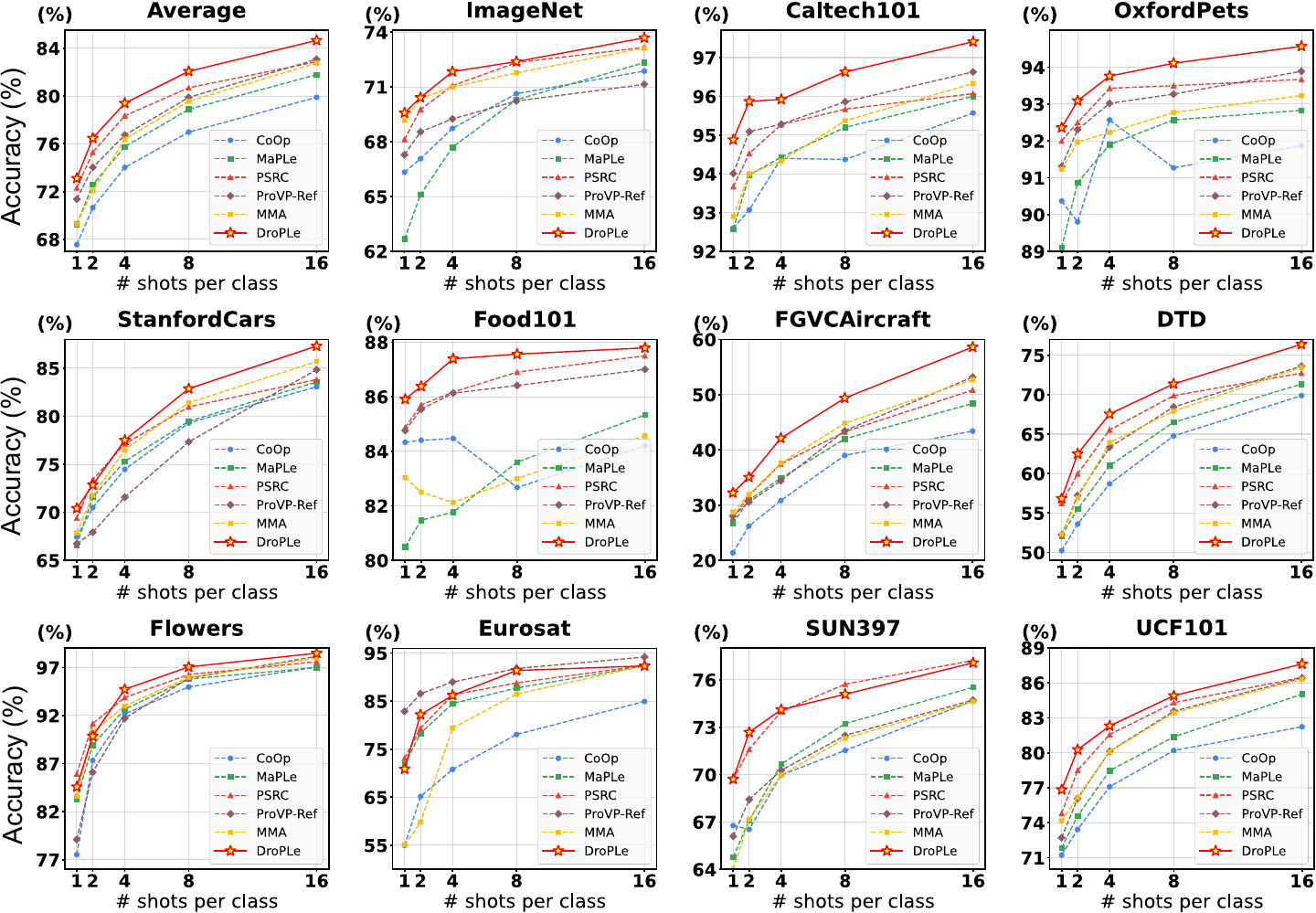}
    \vspace{-10pt}
    \caption{\textbf{Few-shot classification.} Performance comparison on 11 datasets across different shot settings.}
    \label{fig_few-shot2}
    \vspace{-7pt}
\end{figure*}

Considering the structure of Rademacher complexity bounds (e.g., Eq.~\eqref{eq:rad_prod_bound_proof_final_appendix}, where complexity generally increases with terms like $\Phi_i$ which relates to $||\mathbf{q}^{(i)}||_1$), the inequality $\Phi_i(\mathbf{q}_{\text{IWTD}}^{(i)}) \leq \Phi_i(\mathbf{q}_{\text{VTD}}^{(i)})$ suggests that IWTD leads to a smaller or equal Rademacher complexity:
\begin{align}
\hat{\mathcal{R}}_{\mathcal{X}}(\mathbb{F}_{\text{IWTD}}) \leq \hat{\mathcal{R}}_{\mathcal{X}}(\mathbb{F}_{\text{VTD}}) \label{eq:final_rademacher_comp_l1_appendix}.
\end{align}
Substituting Eq.~\eqref{eq:final_rademacher_comp_l1_appendix} into the general bound Eq.~\eqref{eq:std_gen_bound} indicates that IWTD can yield a tighter generalization bound than VTD when VTD is tuned for comparable empirical performance.
\end{proof}
Proposition.~\ref{prop:iwtd_rademacher_comparison} shows that IWTD can achieve tighter generalization bounds than VTD with comparable empirical risk, due to lower empirical Rademacher complexity.

\subsection{A.2. Theoretical Analysis of Residual Entropy Regularization}

\noindent\textbf{Theorem 1 (Complementary Design of Residual Entropy Regularization).} 
Residual entropy regularization operates complementarily to importance weighted token dropout by targeting different representational properties, enabling effective joint optimization.

\textit{Proof.} 
Let $z^o$ denote original features, $z^d$ denote post-dropout features, and $z^r = \frac{z^d - \lambda z^o}{1-\lambda}$ represent the residual component.
The total loss decomposes as:
\begin{equation}
\mathcal{L} = \mathcal{L}_{CE}(z^d) + \mathcal{L}_{RE}(z^r),
\end{equation}
where $\mathcal{L}_{CE}$ depends on $z^d$ and $\mathcal{L}_{RE}$ depends on $z^r$.
The gradient contributions satisfy:
\begin{align}
\frac{\partial \mathcal{L}_{CE}}{\partial p_j} &= \frac{\partial \mathcal{L}_{CE}}{\partial z^d} \cdot \frac{\partial z^d}{\partial p_j} \\
\frac{\partial \mathcal{L}_{RE}}{\partial z^r} &= \nabla \mathcal{H}(p(y|z^r)).
\end{align}
Due to the linear residual relationship where $z^o$ is frozen ($\frac{\partial z^o}{\partial \theta} = 0$), the two loss terms target different aspects: $L_{CE}$ optimizes task performance while $L_{RE}$ constrains residual entropy, enabling complementary optimization effects. $\square$

\noindent\textbf{Theorem 2 (Information-Theoretic Guarantee).} 
Residual entropy regularization ensures that dropout-induced perturbations contain no class-discriminative information.

\textit{Proof.} 
Let $Y$ denote the ground-truth class labels with $K$ classes. The mutual information $\mathcal{I}(Y; z^r)$ measures the statistical dependence between labels and residual components. By the definition of mutual information:
\begin{equation}
\mathcal{I}(Y; z^r) = \mathcal{H}(Y) - \mathcal{H}(Y|z^r).
\end{equation}
Therefore, maximizing conditional entropy $\mathcal{H}(Y|z^r)$ is equivalent to minimizing mutual information $\mathcal{I}(Y; z^r)$:
\begin{equation}
\max \mathcal{H}(Y|z^r) = \max \mathcal{H}(Y) - \mathcal{I}(Y; z^r) \Rightarrow \min \mathcal{I}(Y; z^r).
\end{equation}
When $\mathcal{I}(Y; z^r) \rightarrow 0$, the residual $z^r$ becomes statistically independent of labels $Y$:
\begin{equation}
p(y|z^r) \approx p(y) = \frac{1}{K}.
\end{equation}
This guarantees that residual components $z^r$ contain only non-discriminative information, preserving semantic alignment. $\square$

\textbf{Corollary.} IWTD controls token retention through dropout probabilities $p_j$, while residual entropy constrains perturbation distributions via $\max \mathcal{H}(Y|z^r)$. These mechanisms target different aspects of the learned representations and work synergistically.

\begin{table*}
  \centering
  \caption{\textbf{Few-shot classification across different VLM architectures.} Comparison with several few-shot learning methods on 11 datasets. Best results highlighted in \colorbox[HTML]{d9e6fa}{\textbf{first}}, \colorbox[HTML]{ebf0fa}{second}.}
  \renewcommand{\arraystretch}{0.96}
  \scalebox{1}{
    \begin{tabular}{llccccccccccc>{\columncolor[HTML]{f8f8f8}}c}
    \toprule
          &       & \multicolumn{12}{c}{\textbf{16-Shot Classification}} \\
\cmidrule{3-14}    \textbf{Model} & \multicolumn{1}{l}{\textbf{Method}} & \rotatebox{60}{\textbf{ImgNet}} & \rotatebox{60}{\textbf{Cal101}} & \rotatebox{60}{\textbf{Pets}}   & \rotatebox{60}{\textbf{Cars}}  & \rotatebox{60}{\textbf{Flowers}} & \rotatebox{60}{\textbf{Food}}  & \rotatebox{60}{\textbf{FGVC}}  & \rotatebox{60}{\textbf{SUN}}  & \rotatebox{60}{\textbf{DTD}}  & \rotatebox{60}{\textbf{SAT}}  & \rotatebox{60}{\textbf{UCF}}  & \rotatebox{60}{\textbf{\textit{Avg.}}}  \\
    \midrule
    \midrule
    \multirow{6}[2]{*}{ViT-B/32} & CoOp$_{\text{(IJCV'22)}}$ & 66.7  & 95.0  & 89.4  & 71.4  & 93.1  & 79.8  & 31.1  & 72.1  & 64.3  & 80.6  & 78.2  & 74.7  \\
          & CoCoOp$_{\text{(CVPR'22)}}$ & 66.0  & 94.3  & 91.0  & 64.6  & 82.5  & 81.9  & 22.6  & 69.8  & 59.7  & 70.4  & 75.3  & 70.7  \\
          & KgCoOp$_{\text{(CVPR'23)}}$ & 65.4  & 94.4  & 90.8  & 67.3  & 86.1  & 81.7  & 23.7  & 71.0  & 65.1  & 70.1  & 77.5  & 72.1  \\
          & MaPLe$_{\text{(CVPR'23)}}$ & 66.7  & 95.1  & \cellcolor[HTML]{ebf0fa}91.7  & 66.9  & 89.0  & \cellcolor[HTML]{ebf0fa}82.1  & 28.0  & 72.0  & 63.4  & \cellcolor[HTML]{ebf0fa}83.3  & 77.3  & 74.1  \\
          & MMA$_{\text{(CVPR'24)}}$ & \cellcolor[HTML]{ebf0fa}68.0  & \cellcolor[HTML]{ebf0fa}95.6  & 91.5  & \cellcolor[HTML]{ebf0fa}73.5  & \cellcolor[HTML]{ebf0fa}94.3  & 81.4  & \cellcolor[HTML]{ebf0fa}34.0  & 74.0  & \cellcolor[HTML]{ebf0fa}68.9  & 80.1  & \cellcolor[HTML]{ebf0fa}81.7  & \cellcolor[HTML]{ebf0fa}76.6  \\
          \midrule
          & DroPLe$_{\text{(Ours)}}$ & \cellcolor[HTML]{d9e6fa}\textbf{68.7}  & \cellcolor[HTML]{d9e6fa}\textbf{96.7}  & \cellcolor[HTML]{d9e6fa}\textbf{92.8}  & \cellcolor[HTML]{d9e6fa}\textbf{75.6}  & \cellcolor[HTML]{d9e6fa}\textbf{94.8}  & \cellcolor[HTML]{d9e6fa}\textbf{84.5}  & \cellcolor[HTML]{d9e6fa}\textbf{40.3}  & \cellcolor[HTML]{d9e6fa}\textbf{76.3}  & \cellcolor[HTML]{d9e6fa}\textbf{71.7}  & \cellcolor[HTML]{d9e6fa}\textbf{83.5}  & \cellcolor[HTML]{d9e6fa}\textbf{83.4}  & \cellcolor[HTML]{d9e6fa}\textbf{78.9}  \\
    \midrule
    \midrule
    \multirow{6}[2]{*}{ViT-L/14} & CoOp$_{\text{(IJCV'22)}}$ & 78.1  & 97.5  & 94.5  & 87.4  & \cellcolor[HTML]{ebf0fa}98.6  & 90.2  & 53.0  & 77.9  & 73.7  & \cellcolor[HTML]{d9e6fa}\textbf{86.7}  & 86.7  & 84.0  \\
          & CoCoOp$_{\text{(CVPR'22)}}$ & 77.8  & 97.4  & 95.4  & 82.7  & 95.3  & 91.9  & 45.2  & 76.7  & 71.4  & 79.8  & 85.2  & 81.7  \\
          & KgCoOp$_{\text{(CVPR'23)}}$ & 76.8  & 97.4  & 95.3  & 83.2  & 96.4  & 91.7  & 47.5  & 76.7  & 73.6  & 83.6  & 86.4  & 82.6  \\
          & MaPLe$_{\text{(CVPR'23)}}$ & 78.4  & 97.2  & 95.4  & 83.6  & 97.4  & 92.0  & 46.3  & 78.8  & 72.7  & \cellcolor[HTML]{ebf0fa}85.4  & 86.5  & 83.1  \\
          & MMA$_{\text{(CVPR'24)}}$ & \cellcolor[HTML]{ebf0fa}79.9  & \cellcolor[HTML]{ebf0fa}97.6  & \cellcolor[HTML]{ebf0fa}95.5  & \cellcolor[HTML]{ebf0fa}88.0  & 98.4  & \cellcolor[HTML]{ebf0fa}92.0  & \cellcolor[HTML]{ebf0fa}56.4  & \cellcolor[HTML]{ebf0fa}80.2  & \cellcolor[HTML]{ebf0fa}75.8  & 76.3  & \cellcolor[HTML]{ebf0fa}88.0  & \cellcolor[HTML]{ebf0fa}84.4  \\
          \midrule
          & DroPLe$_{\text{(Ours)}}$ & \cellcolor[HTML]{d9e6fa}\textbf{80.4}  & \cellcolor[HTML]{d9e6fa}\textbf{98.3}  & \cellcolor[HTML]{d9e6fa}\textbf{96.3}  & \cellcolor[HTML]{d9e6fa}\textbf{89.1}  & \cellcolor[HTML]{d9e6fa}\textbf{98.9}  & \cellcolor[HTML]{d9e6fa}\textbf{93.5}  & \cellcolor[HTML]{d9e6fa}\textbf{60.8}  & \cellcolor[HTML]{d9e6fa}\textbf{82.6}  & \cellcolor[HTML]{d9e6fa}\textbf{76.7}  & 84.6  & \cellcolor[HTML]{d9e6fa}\textbf{88.9}  & \cellcolor[HTML]{d9e6fa}\textbf{86.4}  \\
    \bottomrule
    \end{tabular}%
    }
  \label{tab:different_vits}%
  \vspace{-3pt}
\end{table*}%

\section{B. Additional Experiments}
\subsection{B.1. Cross-dataset Evaluation}\label{cross-dataset_eval}
We evaluate the cross-dataset generalization capability on 10 diverse target datasets. As shown in Table~\ref{tab:cross_dataset}, DroPLe achieves 67.28\% average accuracy, outperforming recent works like TAP (66.56\%), DeKgTCP (66.61\%) and CoPrompt (67.00\%). Our method shows notable improvements on some dataset, with a +1.31\% over DeKgTCP on FGVCAircraft (26.36\%) and +0.73\% over CoPrompt on EuroSAT (52.63\%). The results reveal that our importance-guided token dropout effectively balances feature preservation and diversity, leading to better cross-dataset transferability.

\subsection{B.2. Complete Few-shot Classification Results}
We present the comprehensive few-shot classification results across all 11 datasets in Fig.~\ref{fig_few-shot2}. DroPLe demonstrates strong few-shot learning capabilities, particularly excelling on several challenging datasets. At 16 shots per class, DroPLe achieves notable performance with 73.69\% on ImageNet, 97.41\% on Caltech101, 94.57\% on OxfordPets, 87.31\% on StanfordCars, 87.79\% on Food101, 58.62\% on FGVCAircraft, 76.38\% on DTD, 98.49\% on Flowers102, 77.08\% on SUN397, and 87.62\% on UCF101, respectively. The method shows competitive results across different shot settings, with our importance weighted token dropout providing effective regularization in data-limited scenarios. These results validate that our adaptive dropout strategy successfully maintains crucial semantic information while introducing beneficial diversity for improved generalization in few-shot learning.

\subsection{B.3. Cross-Architecture and Adapter Evaluation}\label{app:cross_arch}

\noindent\textbf{Different VLM Architectures.} To validate the applicability of our method across different VLMs, we conducted few-shot classification experiments on CLIP with ViT-B/32 and ViT-L/14 architectures. As shown in Table~\ref{tab:different_vits}, compared to several prompt learning methods in few-shot scenarios, our approach achieved the highest average accuracy across 11 datasets on both backbones, reaching 78.9\% and 86.4\% respectively, outperforming the second-best method by 2.3\% and 2.0\%. This further demonstrates the robustness of our dropout prompt learning when dealing with limited samples.

\noindent\textbf{Adapter-based Methods.} To validate DroPLe's generalizability beyond prompt learning, we evaluate its effectiveness with adapter-based methods on ImageNet 16-shot classification. As shown in Table~\ref{tab:adapter_methods}, DroPLe consistently improves various adapter-based methods across ViT-B/16 and ViT-B/32 architectures. Specifically, DroPLe achieves improvements of +0.36\%/+0.72\% with Tip-Adapter~\cite{zhang2022tip}, +0.46\%/+0.15\% with DAC-V~\cite{gondal2024domain}, and +0.15\%/+0.41\% with DAC-VT~\cite{gondal2024domain} on ViT-B/16/ViT-B/32 respectively. For Tip-Adapter-F~\cite{zhang2022tip}, DroPLe maintains comparable performance on ViT-B/16 while achieving +0.31\% improvement on ViT-B/32. These results demonstrate that our importance-weighted token dropout mechanism effectively enhances various parameter-efficient adaptation approaches, confirming the broad applicability of our method across different adaptation paradigms.

\begin{table}[htbp]
  \centering
  \caption{16-shot classification performance of \textbf{adapter-based methods} on ImageNet.}
  \label{tab:adapter_methods}
  \vspace{-3pt}
  \begin{minipage}{0.22\textwidth}
    \centering
    \small
    \setlength{\tabcolsep}{1.5pt}
    \scalebox{0.95}{
    \begin{tabular}{lcc}
    \toprule
    Models & ViT-B/16 & ViT-B/32 \\
    \midrule
    Tip-Adapter & 70.83  & 65.60  \\
    \rowcolor[HTML]{e9f0fb}
    \ \ +DroPLe & 71.19  & 66.32  \\
    \midrule
    Tip-Adapter-F & 73.70  & 78.74  \\
    \rowcolor[HTML]{e9f0fb}
    \ \ +DroPLe & 73.66  & 79.05  \\
    \bottomrule
    \end{tabular}
    }
  \end{minipage}%
  \hspace{13pt}
  \begin{minipage}{0.22\textwidth}
    \centering
    \small
    \setlength{\tabcolsep}{1.5pt}
    \scalebox{0.95}{
    \begin{tabular}{lcc}
    \toprule
    Models & ViT-B/16 & ViT-B/32 \\
    \midrule
    DAC-V & 72.98  & 67.77  \\
    \rowcolor[HTML]{e9f0fb}
    \ \ +DroPLe & 73.44  & 67.92  \\
    \midrule
    DAC-VT & 74.59  & 69.64  \\
    \rowcolor[HTML]{e9f0fb}
    \ \ +DroPLe & 74.74  & 70.05 \\
    \bottomrule
    \end{tabular}%
    }
  \end{minipage}
  \vspace{-5pt}
\end{table}

\subsection{B.4. Additional Visual Analysis}\label{appendix:cam_vis}
To provide additional visual analysis beyond the main paper, we present extended Grad-CAM visualizations comparing our importance weighted token dropout with vanilla dropout across multiple examples. 
As shown in Fig.~\ref{cam_appendix}, our method consistently produces more concentrated and focused activation maps compared to vanilla dropout across diverse image categories. For the car images (left columns), our importance weighted token dropout shows clear attention on the vehicle's structure and distinctive features, while vanilla dropout exhibits more scattered attention patterns. In the flower images (right columns), our method demonstrates precise localization on the flower's center and petals, whereas vanilla dropout shows diffuse activation across less relevant regions. The heat map visualizations clearly indicate that our adaptive token preservation strategy maintains stronger focus on semantically relevant object parts, validating that our multimodal importance-based approach effectively preserves critical visual features while providing regularization benefits.

\section{C. Residual Design: Linear vs. Non-linear}
The linear residual formulation in Eq.~(\ref{eq:residual_def}) enables clean separation of dropout-induced variations from original representations. We analyze this design choice through both theoretical and empirical perspectives.

\noindent\textbf{Invertibility Requirement.} The linear formulation $z^d = \lambda z^o + (1-\lambda)z^r$ maintains invertibility: $z^r = \frac{z^d - \lambda z^o}{1-\lambda}$, allowing precise extraction of the residual component. This property is essential for applying residual entropy regularization $\max \mathcal{H}(Y|z^r)$, as the method requires direct access to $z^r$ for effective regularization.

\noindent\textbf{Non-linear Alternative.} We examine a non-linear alternative where $z^d$ is computed via a multi-layer perceptron that takes concatenated features $[z^o; z^r]$ as input. This approach lacks invertibility, as given $z^d$ and $z^o$, we cannot uniquely recover $z^r$. The ability to directly access $z^r$ is essential for applying residual entropy regularization $\max \mathcal{H}(Y|z^r)$, as established in Theorem 2. Without access to the true residual component, residual entropy regularization becomes ineffective. Moreover, the non-linear transformation introduces significantly higher computational overhead, contradicting the efficiency principles of prompt learning. Even invertible neural networks could theoretically provide access to residuals, they introduce substantial computational overhead due to their requirement for specialized architectures and activation recomputation during backpropagation~\cite{jacobsen2018revnet}.

\begin{figure}[t]
    \centering    \includegraphics[width=1.0\linewidth]{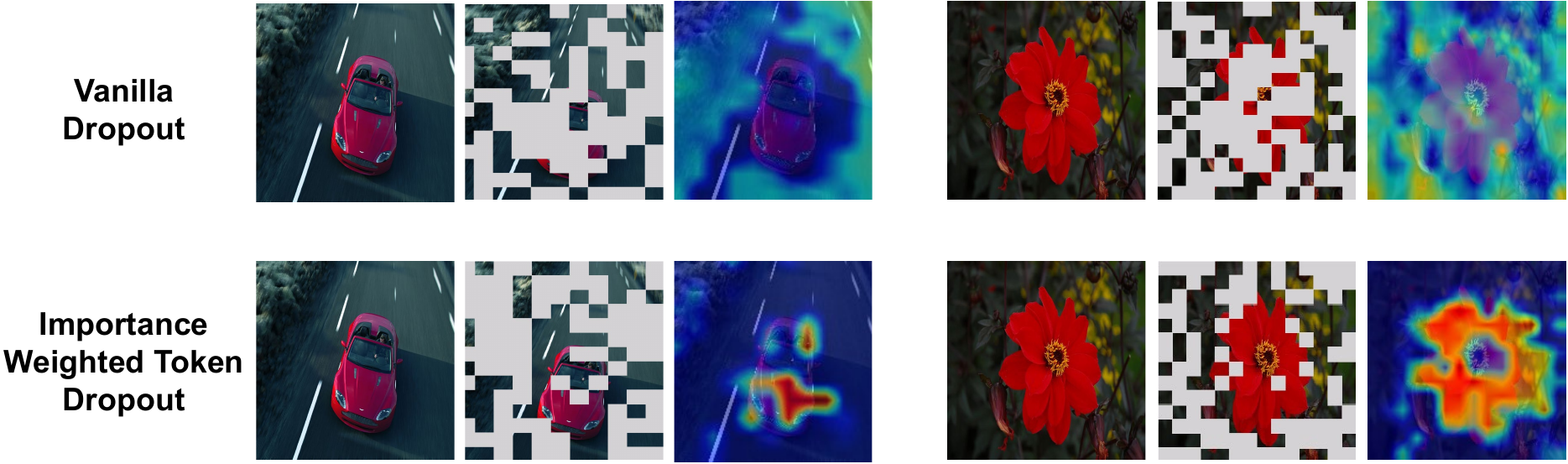}
    \caption{\textbf{Grad-CAM visualizations for different dropout methods.} The first row shows Vanilla dropout, and the second row shows Importance Weighted Token Dropout. Redder colors indicate higher feature attention.}
    \label{cam_appendix}
\end{figure}

\begin{table}[t]
\centering
\caption{Linear vs. non-linear residual formulation comparison.}
\label{tab:linear_vs_nonlinear}
\begin{tabular}{lccc}
\toprule
Method & Base & Novel & HM \\
\midrule
\rowcolor[HTML]{e9f0fb}
Linear & 86.12 & 78.44 & 82.10 \\
Non-linear & 85.41 & 76.82 & 80.89 \\
\bottomrule
\end{tabular}
\end{table}

\begin{algorithm*}[t]
\caption{Dropout Prompt Learning (DroPLe)}
\label{alg:drople_polished} 
\begin{algorithmic}[1] 
    \Require Pre-trained VLM $(\mathcal{F}, \mathcal{G})$, Dataset $D$, Learnable Prompts $\Theta_P = \{P_v, P_t\}$, MaxEpoch $E$, Hyperparameters $\lambda$, $\gamma$, $p_{min}, p_{max}$.
    \State Initialize learnable prompts $\Theta_P$.
    \For{$epoch = 1 \ldots E$}
        \For{\text{each batch} $(x_v, x_t) \in D$}
            \State $U_v, U_t \gets \text{InputTokens}(x_v, x_t, \Theta_P)$. \Comment{Token sequences with prompts}
            \State $U_v^d \gets \text{ApplyIWTD}(U_v, \text{`vision'}, p_{min}, p_{max})$. \Comment{Apply IWTD to visual tokens}
            \State $U_t^d \gets \text{ApplyIWTD}(U_t, \text{`text'}, p_{min}, p_{max})$. \Comment{Apply IWTD to textual tokens}
            \State $z_v^d, z_t^d \gets \mathcal{F}(U_v^d), \mathcal{G}(U_t^d)$. \Comment{Encode features after IWTD}
            \State $z_v^o, z_t^o \gets \mathcal{F}(U_v), \mathcal{G}(U_t)$. \Comment{Encode original features (no dropout)} 
            \State $z_v^r \gets (z_v^d - \lambda z_v^o) / (1-\lambda)$. \Comment{Visual residual component} from (\ref{eq:residual_def})
            \State $z_t^r \gets (z_t^d - \lambda z_t^o) / (1-\lambda)$. \Comment{Textual residual component} from (\ref{eq:residual_def})
            \State $\mathcal{L}_{RE}^v \gets -\mathcal{H}(P(Y | z_v^r, \{\mathcal{G}(\text{text}_k)\}_k))$. \Comment{Visual residual entropy from (\ref{entropy_re})} 
            \State $\mathcal{L}_{RE}^t \gets -\mathcal{H}(P(Y | z_t^r, \{\mathcal{F}(\text{image}_k)\}_k))$. \Comment{Textual residual entropy from (\ref{entropy_re})} 
            \State $\mathcal{L}_{CE} \gets \text{ComputeTaskLoss}(z_v^d, z_t^d, \text{labels})$.
            \State $\mathcal{L}_{total} \gets \mathcal{L}_{CE} + \mathcal{L}_{RE}^v + \mathcal{L}_{RE}^t$. 
            \State Update prompts $\Theta_P$ using $\nabla_{\Theta_P} \mathcal{L}_{total}$.
        \EndFor
    \EndFor
\end{algorithmic}
\end{algorithm*}

\noindent\textbf{Empirical Comparison.} We compare linear and non-linear formulations on base-to-novel generalization. As shown in \ref{tab:linear_vs_nonlinear}, the linear approach achieves 82.10\% harmonic mean, outperforming the non-linear alternative by 1.21\%. This performance gap confirms that the loss of invertibility in non-linear formulations prevents effective residual entropy regularization. When the true residual component $z^r$ cannot be precisely isolated, the entropy maximization objective fails to remove class-discriminative information from dropout perturbations, leading to degraded generalization performance.

\section{D. Comparison with Unimodal Adaptive Dropout Methods}
While our IWTD method builds upon adaptive dropout techniques, it extends beyond traditional unimodal approaches by jointly considering both intra-modal and inter-modal token importance. To validate the effectiveness of this multimodal design, we compare IWTD with representative unimodal adaptive dropout methods applied to VLMs.

\noindent\textbf{Baseline Adaptations.} We implement three unimodal adaptive dropout variants: (1) \textbf{AD-DROP-Uni}: applies gradient-based attribution~\cite{yang2022ad} separately to visual and textual tokens within each modality; (2) \textbf{Group-Wise-Uni}: adapts the density-based grouping method~\cite{ke2020group} by applying PCA projection and grid-based density estimation independently to visual and textual token representations; (3) \textbf{StandOut-Uni}: implements learnable dropout probabilities~\cite{ba2013adaptive} for each modality separately. Critically, all these methods operate within individual modalities and cannot capture cross-modal semantic alignment.

\noindent\textbf{Results and Analysis.} Table~\ref{tab:unimodal_comparison} presents the comparison results on base-to-novel generalization. While unimodal adaptive methods show improvements over vanilla dropout, they significantly underperform our IWTD approach. Notably, Group-Wise-Uni achieves only 78.89\% HM, demonstrating that intra-modal feature density alone is insufficient for VLM regularization. AD-DROP-Uni performs better at 79.25\% HM but still falls short of IWTD's 82.10\% HM.

The superior performance of IWTD can be largely attributed to its cross-modal awareness. While unimodal methods may drop tokens critical for image-text alignment (e.g., dropping visual tokens of ``car'' when the text mentions ``car''), IWTD considers cross-modal semantic relationships through $S_{cross}$ with intra-modal importance scores, ensuring semantically aligned tokens across modalities are preserved together. This comprehensive multimodal coordination is not achievable with purely intra-modal adaptive dropout strategies.

\noindent\textbf{Ablation on Cross-modal Component.} To isolate the contribution of cross-modal awareness, we compare IWTD with and without the cross-modal attention score $S_{cross}$. Removing $S_{cross}$ (using only $S_{self}$ and $S_{cls}$) reduces performance to 81.14\% HM, demonstrating the importance of cross-modal importance evaluation in distinguishing our approach from traditional adaptive dropout methods.

\begin{table}[t]
\centering
\caption{Comparison with unimodal adaptive dropout methods on base-to-novel generalization.}
\vspace{-3pt}
\label{tab:unimodal_comparison}
\renewcommand{\arraystretch}{0.95}
\begin{tabular}{lccc|c}
\toprule
Method & Base & Novel & HM & $\Delta$HM \\
\midrule
Dropout & 82.71 & 73.45 & 77.81 & - \\
AD-DROP-Uni & 83.95 & 75.25 & 79.25 & +1.44 \\
Group-Wise-Uni & 83.51 & 74.84 & 78.89 & +1.08 \\
StandOut-Uni & 83.79 & 75.11 & 79.09 & +1.28 \\
\midrule
\rowcolor[HTML]{e9f0fb}
\textbf{IWTD (Ours)} & \textbf{86.12} & \textbf{78.44} & \textbf{82.10} & \textbf{+4.29} \\
\bottomrule
\end{tabular}
\vspace{-3pt}
\end{table}

\section{E. Algorithm}
We present our Dropout Prompt Learning algorithm in Algorithm~\ref{alg:drople_polished}. After initializing prompts in step 1, each training iteration proceeds as follows. First, in steps 4 to 8, we generate input tokens, apply importance weighted token dropout, and encode both dropout-modified and original features. Next, in steps 9 to 12, we compute feature residuals and their associated entropy losses for residual entropy regularization. Finally, in steps 13 to 15, we calculate a task-specific loss, combine it with the entropy losses to form the total training objective, and update the learnable prompts.

\end{document}